%% file: graph_sequences.tex
\documentclass[12pt]{article}
\usepackage{fullpage}

\usepackage{amsmath,amsthm,amssymb,mathtools,amsfonts,bbm}

\usepackage{graphicx,caption,color,comment,hyperref,natbib}
\usepackage{algorithm}
\usepackage[noend]{algpseudocode}
\usepackage{graphicx,comment,boldline,natbib,pifont,color}
\usepackage{amsmath,amsthm,amssymb,mathtools,amsfonts,bbm}
\usepackage{exscale}

\theoremstyle{definition}

\newtheorem{example}{Example}[section]
\newtheorem{theorem}{Theorem}[section]
\newtheorem{lemma}{Lemma}[section]
\newtheorem{corollary}{Corollary}[section]

\DeclareMathOperator{\adj}{\textbf{A}}
\DeclareMathOperator{\gro}{\textit{g}}
\DeclareMathOperator{\no}{\textit{n}}
\DeclareMathOperator{\noc}{\textit{k}}
\DeclareMathOperator{\wrs}{\textit{W}} 
\DeclareMathOperator*{\argmin}{argmin} 
\DeclareMathOperator*{\argmax}{argmax}
\DeclareMathOperator{\E}{\mathbb{E}}
\DeclareMathOperator{\pri}{\textit{o}}
\DeclarePairedDelimiter\abs{\lvert}{\rvert}%
\DeclarePairedDelimiter\norm{\lVert}{\rVert}

\newcommand{\cmark}{\ding{51}}
\newcommand{\xmark}{\ding{55}}
\newcommand{\pzo}{f}
\newcommand{\poo}{h}
\newcommand{\per}{\sigma}
\newcommand{\trn}{'}
\newcommand{\maj}{M}
\newcommand{\asor}{\alpha}
\newcommand{\y}{Y} 
\newcommand{\q}{q} 
\newcommand{\p}{\rho} 
\newcommand{\wt}{\varpi} 

\newcommand{\eps}{\epsilon} 

\title{Block-Structure Based Time-Series Models For Graph Sequences}
\author{Mehrnaz Amjadi and Theja Tulabandhula\\
University of Illinois at Chicago\\
\{mamjad2,theja\}@uic.edu}
\date{}

\begin{document}
\maketitle

\begin{abstract}
\input{sec_abstract}
\end{abstract}

\input{sec_introduction}
\input{sec_group_fixed}
\input{sec_group_fixed_algo}
\input{sec_group_changing}
\input{sec_group_changing_algo}
\input{sec_experiments}

\input{sec_conclusion}

\bibliographystyle{plainnat}
\bibliography{bib_graph_sequences}

\appendix
\input{sec_appendix_proofs}

\end{document}

%% file: sec_abstract.tex
Although the computational and statistical trade-off for modeling single graphs, for instance, using block models is relatively well understood, extending such results to sequences of graphs has proven to be difficult. In this work, we take a step in this direction by proposing two models for graph sequences that capture: (a) link persistence between nodes across time, and (b) community persistence of each node across time. In the first model, we assume that the latent community of each node does not change over time, and in the second model we relax this assumption suitably. For both of these proposed models, we provide statistically and computationally efficient inference algorithms, whose unique feature is that they leverage community detection methods that work on single graphs. We also provide experimental results validating the suitability of our models and methods on synthetic and real instances.

%% file: sec_introduction.tex
\section{Introduction}\label{sec:introduction}

There are many statistical approaches to modeling a single network (or graph), while techniques to work with graph sequences have been relatively nascent. For instance, the popular stochastic block model (SBM), a random graph model with planted communities\footnote{We sidestep the concern of formally defining what a community means here.}, has been extremely popular in various disciplines such as statistical physics, statistics, and computer science, because it can capture certain desirable properties of real-world networks~\cite{abbe_community_2017}. The SBM and its variants assume that the interaction between nodes (members) in a network is only based on the identities of the latent communities that they belong to. As such, the model has been used to study online and offline social networks, protein-protein interaction networks, ecological and communication networks, and has also been used to build recommendation systems and identify pathways in biological networks. Recent research has focused heavily on efficient detectability and recovery of the latent community structure as well as the model parameters for the SBM, with near optimal algorithms in both the statistical and computational sense~\cite{abbe_recovering_2015,gao_achieving_2015,xu_optimal_2017}. Building on this line of research, our paper attempts to define models for \emph{sequences} of graphs that resemble the SBM (upon suitable marginalization), and provides efficient inference algorithms for inferring the model parameters. These new models, which work with sequences of graphs, can find use in applications that involve controlling or modulating network growth across time.

Our original goal was to propose new models for graph sequences based on latent block structure, while keeping the following aspects in mind: the first is whether one can reuse the computational advances made for the SBM while designing our inference procedures, and the second is whether one can parallel the statistical guarantees for detectability and recovery available for the SBM in our setting. In this paper, we partially achieve both goals.  For instance, in~\cite{abbe_detection_2015}, the authors propose a computationally efficient community detection algorithm (a linearized acyclic belief propagation method), and this method can be easily used in our inference methods. We can also use the well known spectral clustering technique. The reason we are able to use many of these algorithms, which are applicable for the SBM, is that we carefully take into account the Markovian dynamics involved with how links/edges and community memberships change over time in our proposed models. A consequence of this choice is that marginalized observations retain the SBM structure, and hence efficient methods such as spectral clustering can be utilized. In the paper, we also discuss how such a careful design does not make our models less flexible or narrow. In fact, recent works such as~\cite{barucca_detectability_2017} also propose similar Markovian models, but fall back on the general Expectation-Maximization (EM) algorithmic template for  inference that does not allow reuse of estimation methods designed for the SBM. 

To be precise, we propose two general models for sequences of graphs as shown in Figure \ref{fig_model}, where the marginal distributions of graphs retain an SBM structure. The first model focuses on \emph{link persistence}, where links observed at a given time step depend on history, but the latent community memberships of nodes remain unchanged throughout. In this setting, we capture two special cases of link evolution. In the first, links evolve over time based on a conditional Bernoulli distribution, and in the second, links either remain the same as before with some probability $\xi > 0$, or they are generated according to a fixed unknown SBM model with probability $(1-\xi)$ . In terms of analogy, this is akin to a lazy random walk on a suitable Markov chain. Our proposed inference procedure under this model stitches together the estimated community memberships from multiple different SBM estimates, which are potentially \emph{permuted} with respect to each other. Once we infer the latent communities in this model, we use maximum likelihood estimation (MLE) to estimate the remaining parameters. In the special cases described above, we also develop systems of nonlinear equations that relate estimates made at different time indices to the model parameters. This latter approach can potentially be more transparent about errors in the parameter estimates compared to performing a non-linear MLE optimization. In addition to correctness claims, we provide a finite sample statistical guarantee for one of the intermediate steps using matrix concentration results (we also discuss how analysis techniques developed in the literature on clustering under noise can be reused in our setting). 

In the second model, we capture \emph{community persistence} (in addition to link persistence) by assuming that between two consecutive graph snapshots, a small \emph{non-stochastic} number of nodes (\emph{minority}) change their communities in a stochastic manner. This is a fairly reasonable assumption to make if one can take measurements of graphs at increasingly frequent intervals of time such that most of the edges and other network properties including community memberships remain the same. In a sense, our partially-stochastic model can be considered as an alternative to the fully-stochastic approach of~\cite{xu_stochastic_2015,barucca_detectability_2017,ghasemian_detectability_2016,rastelli2017exact}, where every node can change its latent community with some probability. Under the assumption that we know the majority and minority labels of nodes (this is much coarser information than knowing the community memberships themselves), we propose an inference procedure to infer the community memberships of most nodes across time, building on the previous strategy while accounting for community changes. We extensively motivate when such coarser information is available in Section~\ref{sec:group-changing-model}.

\begin{figure}	
\begin{center}
\includegraphics[width=0.42\columnwidth]{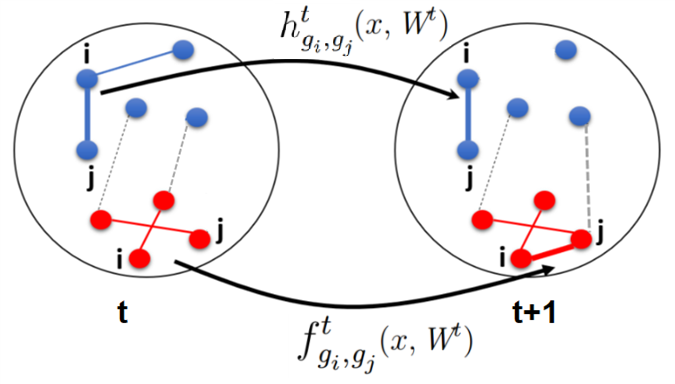}{}
	\includegraphics[width=0.45\columnwidth]{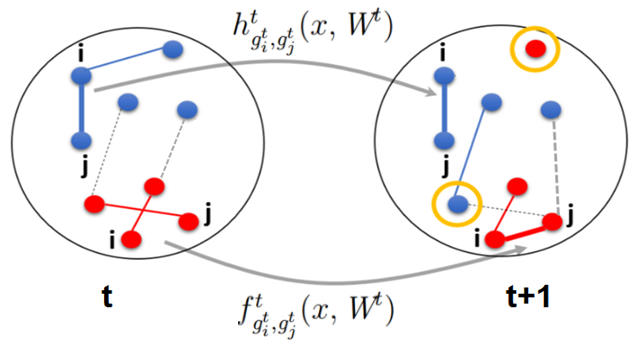}
\caption{Left: Illustration of link persistence (Section~\ref{sec:group-fixed-model}). Right: Illustration of link and community persistence (Section~\ref{sec:group-changing-model}). Functions $\pzo^t_{\gro_i,\gro_j}(x,\wrs^t)$ and $\poo^t_{\gro_i,\gro_j}(x, \wrs^t)$ model the Markovian dynamics in Section~\ref{sec:group-fixed-model}. Functions $\pzo^t_{\gro_i^t,\gro_j^t}(x,\wrs^t)$ and $\poo^t_{\gro_i^t,\gro_j^t}(x, \wrs^t)$ model the Markovian dynamics in Section~\ref{sec:group-changing-model}. } \label{fig_model}
\end{center}
\end{figure}

The notions of link and community persistence have been captured in varied previous works such as~\cite{anagnostopoulos_community_2016,ghasemian_detectability_2016,herlau_modeling_2013,pensky_spectral_2017,xing_state-space_2010,xu_dynamic_2014,han_consistent_2015,xu_stochastic_2015,yang_detecting_2011,zhang_random_2016}, and ~\cite{barucca_detectability_2017} among others. While some of them are concerned with detectability questions (i.e., parameter regimes where their model is identifiable), others have proposed inference algorithms (based on EM, Markov Chain Monte Carlo, belief propagation, MLE, spectral clustering, etc.) that tend to be computationally inefficient (some of the largest sizes of graphs they can handle is $\sim 1000$). One of the hurdles that hampers scalability in these time series methods is that many of these attempt to jointly estimate community memberships (a hard combinatorial problem) and the model dynamics at the same time (sometimes this is a consequence of the model dynamics). On the other hand, due to our modeling choices, the corresponding inference algorithms are much more leaner and faster. They rely on subroutines that perform community membership inference on single graphs assuming an SBM structure. This allows them to side-step any joint estimation issues by efficiently solving multiple single graph community estimation problem instances (this can be done in \emph{parallel}) and then stitching the inferred memberships together. While the paper does not run extensive experiments comparing the scalability of our methods, it is evident that as better methods become available for solving a single graph SBM instance, our inference methods for graph sequences can also become faster. Note that because our methods rely on SBM subroutines, the detectability thresholds levels at which they work also apply to our inference procedures. No additional investigation of such thresholds is necessary, while for other models and algorithms, one has to analyze these identifiability regimes on a case-by-case basis.

Among these related works, we discuss a few closest ones briefly. \cite{zhang_random_2016} propose a degree-corrected model that captures link persistence assuming a conditional Poisson distribution on the number of edges between each pair of nodes, and builds on the EM framework to recover parameters. 
\cite{ghasemian_detectability_2016} propose a particular \emph{symmetric}  model (this assumption reduces the number of parameters drastically) in which community memberships of nodes persist or uniformly change to other values based on natural Markovian dynamics, whereas the links do not persist. They also describe a belief propagation (BP) based algorithm for inference that achieves limited scalability. A model which is very similar to ours is presented by \cite{barucca_detectability_2017}, who capture both community and link persistence. In their model, either the community membership of each node persists with some probability, or it changes uniformly at random. Also, the links are either retained from one time step to another with some probability, or are generated afresh. The authors focus on detectability issues under various parameter regimes, which is different from our goal of designing efficient inference algorithms. A BP based inference algorithm is proposed for parameter recovery. Unlike their approach, we capture a more realistic property that an \emph{arbitrary} but small number of nodes may change their community memberships from one time step to another. 

In their model, \cite{pensky_spectral_2017} assume that the probabilities governing the Markovian dynamics of communities and links are smooth across time, and some nodes can switch their community memberships. In contrast to our work where we start with community estimation, they first estimate the parameters for their dynamics using tensor-based methods and subsequently use spectral clustering to find the community memberships. In \citet{rastelli2017exact} the author extends the stochastic block transition model (SBTM) model of \citet{xu_stochastic_2015} by bringing in a Bayesian hierarchical structure on the model's edge generation dynamics, and proposing a Bayesian inference framework. In doing so, the author is able to integrate out many parameters leading to a slightly more tractable likelihood function (they call it the 'integrated completed likelihood'). Unfortunately, similar to the Kalman filtering (followed by a local search heuristic) approach for SBTM, they also end up devising a  greedy heuristic to estimate the changing community memberships and other parameters. While they obtain an attractive property that the number of communities can also be learned from data, the procedure is not neither computationally efficient (as is the one by \citet{xu_stochastic_2015}), nor does it come with any strong statistical guarantees. It also loses the desirable property of the marginalized graph snapshots retaining a SBM structure. In contrast, by defining our models carefully, we reuse methods for SBM in the first stage of our inference method, which take care of the combinatorial nature of community estimation, allowing for better overall computational efficiency. \citet{rastelli2017choosing} is a precursor to \citet{rastelli2017exact} that introduces a Bayesian specification to the SBTM model. It specifies priors and proposes inference schemes that come without any computational or statistical guarantees, again by defining appropriate integrated competed likelihood functions.

Finally, we also note the works of \citet{han_consistent_2015} and ~\cite{Bhattacharyya2017}, who suggest the use of Spectral-Mean that estimates communities using an averaged adjacency matrix, and give comprehensive performance bounds. This method can be viewed as an alternative to our aggregation methods, so we experimentally examine how Spectral-Mean performs compared to our proposed methods. The models in these works assume independent generation of edges across time in contrast to our Markovian dynamics, making them quite limited. A comparison with select related work is presented in Table \ref{Compare_Lit} for ease of reference.  

	\begin{table}
	\centering
	\caption{Comparison of this work to previous graph sequence literature}\label{Compare_Lit}
	\label{tab:summary}
	\resizebox{\textwidth}{!}{
	\begin{tabular}{cccccc}
	\hlineB{2}
    \textbf{Papers} & Community  & Link & Approach & Theory & Dataset(s) \\
    			  & persistence  & persistence & & & \\
	\hlineB{2}
\citep{han_consistent_2015} & \textcolor{red}{\xmark} & \textcolor{red}{\xmark} & Maximum Likelihood Estimation & Consistency & MIT Reality Mining\\
\citep{Bhattacharyya2017} & \textcolor{red}{\xmark} & \textcolor{red}{\xmark} & Spectral clustering & Convergence rates & -\\
	\citep{xu_dynamic_2014} & \textcolor{green}{\cmark} & \textcolor{green}{\cmark} & Kalman Filtering+Local Search & - & MIT Reality Mining, Enron emails\\
    	\citep{xu_stochastic_2015} & \textcolor{green}{\cmark} & \textcolor{green}{\cmark} & Kalman Filtering+Local Search & - & Facebook wall posts\\    
	\citep{zhang_random_2016} & \textcolor{red}{\xmark} & \textcolor{green}{\cmark} & Expectation Maximization & - & Internet AS graphs, Friendship networks\\
	\citep{ghasemian_detectability_2016} & \textcolor{green}{\cmark} & \textcolor{red}{\xmark} & Expectation Maximization & Detectability thresholds & Synthetic\\
	\citep{barucca_detectability_2017} & \textcolor{green}{\cmark} & \textcolor{green}{\cmark} & Expectation Maximization & Detectability thresholds & Synthetic\\
	\hline
	\textit{This work} & \textcolor{green}{\cmark} & \textcolor{green}\cmark & Aggregating SBM subroutines + MLE & Correctness/Stage-wise convergence rates
    & Enron emails, Facebook friendships\\
	\hlineB{2}
	\end{tabular}
	}
	\end{table}
    
To summarize, our contributions are as follows: (a) we propose two parsimonious models for sequences of graphs that capture link and community persistence, (b) we develop new algorithms, which infer the parameters driving the Markovian dynamics for community memberships and links across time, and have the following two salient features: (i) they stitch together the estimated community memberships output by efficient methods designed for the single-graph SBM model by taking permutations and estimation errors into account; and (ii) they use MLE (or in special cases, solve a system of equations akin to the method of moments approach). Finally, we also show theoretical and empirical correctness and performance results of our inference procedures on synthetic and real world datasets.

The paper is organized as follows: Section~\ref{sec:group-fixed-model} describes the first dynamic model that captures link persistence across time (two special cases are also illustrated), and Section~\ref{sec:group-fixed-algo} presents the corresponding inference algorithms. In Section~\ref{sec:group-changing-model}, we describe the second dynamic model, which captures both community and link persistence across time, and the corresponding inference algorithm is presented in Section~\ref{sec:group-changing-algo}. Experiments supporting the efficacy of the inference methods are summarized in Section~\ref{sec:experiments}, and we conclude in Section~\ref{sec:conclusion}. Proofs of various claims are provided in the Appendix.

%% file: sec_group_fixed.tex
\section{Capturing Link Persistence}\label{sec:group-fixed-model}

\subsection{Model}\label{model:fixed}

In the first model, we assume that the latent non-overlapping community memberships of nodes do not change over time. This simplifying property has been considered in the literature before (e.g. \citet{zhang_random_2016,han_consistent_2015,Bhattacharyya2017}). Although we relax the fixed community membership assumption in Section~\ref{sec:group-changing-model}, such an assumption is not implausible if graph observations are made at very fine timescales, i.e., in the regime where the link evolution dynamics are the dominant effect. In line with the previous literature, we  also assume that the stochastic process governing the edge formation and deletion for a pair of nodes is independent of corresponding processes for other pairs of nodes (\citet{han_consistent_2015,Bhattacharyya2017,xu_stochastic_2015,barucca_detectability_2017,rastelli2017exact}). 

Let the number of nodes be $n$ and let $\adj = (\adj^1,...,\adj^T)$ be the sequence of observed binary adjacency matrices (with diagonals zero) at each time step (i.e., the graphs are \emph{undirected}). While the setting with the number of nodes changing across time can also be addressed here, keeping them fixed allows for a cleaner exposition. Let the number of latent communities be $\noc$, and let $k_l$ represent the size of each community $l, \quad l=1,..., \noc$. We will denote the \emph{time-invariant} community assignment of nodes in each time step by an $n$-vector $g = \{\gro_i \in \{1, ..., \noc\} \textrm{ for } 1 \leq i \leq n\}$. Let $\y \in R^{\no\times \no}$ be a cluster matrix such that $\y_{ij}=1$ when $i, j$ belong to the same community (i.e., $\gro_i=\gro_j$) and zero otherwise. Our model will be specified using parameters $W^t_{g_ig_j} \in [0,1]$, which along with other parameters introduced below, specify how edges between node $i$ of community membership $g_i$ forms a link with node $j$ of community membership $g_j$ at time $t$. Let $\wrs^t$ represent the collection of these first set of parameters as a matrix of size $k\times k$ for each time index $t$. 

We will denote the second set of parameters using the generic symbol $x$, and define the temporal Markovian dynamics for links between a pair of nodes $i$ and $j$ as follows: (a) the probability of not having a link in the current time step $t$ given that there was no link in the previous time step $t-1$ is denoted by $p^t_{0\rightarrow 0} = 1-\pzo^t_{\gro_i,\gro_j}$ for some function $\pzo^t_{\gro_i,\gro_j}(x,\wrs^t)$, (b) the probability of having a link in the current time step $t$ given that there was no link in the previous time step $t-1$ is denoted by $p^t_{0\rightarrow 1} = \pzo^t_{\gro_i,\gro_j}$, (c) the probability of not having a link in the current time step $t$ given that there was a link in the previous time step $t-1$ is denoted by $p^t_{1\rightarrow 0} = 1-\poo^t_{\gro_i,\gro_j}$ for some function $\poo^t_{\gro_i,\gro_j}(x,\wrs^t)$, and (d) the probability of having a link in the current time step $t$ given that there was a link in the previous time step $t-1$ is denoted by $p^t_{1\rightarrow 1} = \poo^t_{\gro_i,\gro_j}$. In the following, we use $\lVert B\rVert_F$ and $\norm{B}$ to denote the Frobenius norm and the spectral norm of a generic matrix $B$, and use $B\trn$ to denote its transpose. We use $[\noc]$ to represent the set of elements $1,2,...,\noc.$

The conditional likelihood of observing $\adj$ given the latent community membership vector $g$ is defined as:
\begin{align}\label{MLE_1}
P(\adj|\gro)&= \prod_{i< j}^{\no} \Big[P(\adj_{ij}^1)\times \prod_{t=2}^T P(\adj_{ij}^t|\adj_{ij}^{t-1}, \gro)\Big]\nonumber\\
&=\prod_{i< j}^{\no} \Big[{\wrs^1}_{\gro_i\gro_j}^{\adj_{ij}^1} (1-{\wrs^1}_{\gro_i\gro_j})^{1-\adj_{ij}^1} \times \prod_{t=2}^T\prod_{a,b\in \{0,1\}} P(\adj_{ij}^t=a|\adj_{ij}^{t-1}=b, \gro)^{\mathbbm{1}_{[\adj_{ij}^t=a]}\mathbbm{1}_{[\adj_{ij}^{t-1}=b]}}\Big]\nonumber\\
&=\prod_{i< j}^{\no} \Bigg[{\wrs^1}_{\gro_i\gro_j}^{\adj_{ij}^1} (1-{\wrs^1}_{\gro_i\gro_j})^{1-\adj_{ij}^1}\nonumber\\
& \quad \times \prod_{t=2}^T\Big[{\pzo^t}_{\gro_i,\gro_j}^{[1-\adj_{ij}^{t-1}][\adj_{ij}^{t}]}\times (1-\pzo^t_{\gro_i,\gro_j})^{[1-\adj_{ij}^{t-1}][1-\adj_{ij}^{t}]}\nonumber\\
&\quad\quad\quad\quad \times{\poo^t}_{\gro_i,\gro_j}^{[\adj_{ij}^{t-1}][\adj_{ij}^{t}]}\times (1-\poo^t_{\gro_i,\gro_j})^{[\adj_{ij}^{t-1}][1-\adj_{ij}^{t}]}\Big]\Bigg].
\end{align}
Here $\adj_{ij}^t$ represents the $ij^{th}$ element of adjacency matrix $\adj^t$. Different choices for functions $\pzo^t_{\gro_i,\gro_j}(x,\wrs^t)$ and $\poo^t_{\gro_i,\gro_j}(x, \wrs^t)$ give us different Markovian dynamics that capture link persistence. For instance, if $\wrs^t=\wrs$  for all $t$, then we get the following two special cases. 

\begin{example}\label{fixed-model-eg1} (\textit{Type-I dynamic}) If the link between pair $i, j$ follows a conditional Bernoulli distribution with an addition rate $\mu_{\gro_i,\gro_j} \wrs_{\gro_i\gro_j}$ and a deletion rate $\mu_{\gro_i,\gro_j}$ for all $1 \leq t \leq T$, we have:
\begin{align*}
p^t_{0\rightarrow 0}&=1-\pzo^t_{\gro_i,\gro_j}=1-\mu_{\gro_i,\gro_j}\wrs_{\gro_i,\gro_j},\\ p^t_{0\rightarrow1}&=\pzo^t_{\gro_i,\gro_j}=\mu_{\gro_i,\gro_j}\wrs_{\gro_i,\gro_j},\\
p^t_{1\rightarrow 0}&=1-\poo^t_{\gro_i,\gro_j}=\mu_{\gro_i,\gro_j}, \textrm{ and}\\
p^t_{1\rightarrow 1}&= \poo^t_{\gro_i,\gro_j}=1-\mu_{\gro_i,\gro_j}.
\end{align*}
A similar type of dynamic was also investigated in~\cite{zhang_random_2016} for the general Poisson setting, where they make the addition rate dependent on the deletion rate and parameter $\wrs_{\gro_i\gro_j}$. The conditional likelihood in this case is:
\begin{align}\label{Bernoulli link dynamic}
P(\adj|\gro)&= \prod_{i< j}^{\no} \Big[\wrs_{\gro_i\gro_j}^{\adj_{ij}^1} (1-\wrs_{\gro_i\gro_j})^{1-\adj_{ij}^1}\nonumber\\& \quad \times \prod_{t=2}^T(\mu_{\gro_i,\gro_j} \wrs_{\gro_i,\gro_j})^{[1-\adj_{ij}^{t-1}][\adj_{ij}^{t}]}\times (1-\mu_{\gro_i,\gro_j} \wrs_{\gro_i,\gro_j})^{[1-\adj_{ij}^{t-1}][1-\adj_{ij}^{t}]}\nonumber\\&\quad\times(1-\mu_{\gro_i,\gro_j})^{[\adj_{ij}^{t-1}][\adj_{ij}^{t}]}\times (\mu_{\gro_i,\gro_j})^{[\adj_{ij}^{t-1}][1-\adj_{ij}^{t}]}\Big].
\end{align}
\end{example}

\begin{example}\label{fixed-model-eg2}(\textit{Type-II dynamic}) With probability $\xi$, an edge is copied from the previous time step, and with probability $(1-\xi)$ an edge is generated based on the matrix $\wrs$. That is, $P(\adj_{ij} | g_i,g_j) = \wrs_{\gro_i\gro_j}^{\adj_{ij}} (1-\wrs_{\gro_i\gro_j})^{1-\adj_{ij}}$, and this gives us the following relations:
\begin{align*}
p^t_{0\rightarrow 0}&=1-\pzo^t_{\gro_i,\gro_j}=1- (1-\xi)\wrs_{\gro_i,\gro_j},\\
p^t_{0\rightarrow1}&=\pzo^t_{\gro_i,\gro_j}=(1-\xi)\wrs_{\gro_i\gro_j},\\
p^t_{1\rightarrow 0}&=1-\poo^t_{\gro_i\gro_j}=1-\xi-(1-\xi)\wrs_{\gro_i\gro_j}, \textrm{ and}\\
p^t_{1\rightarrow 1}&= \poo^t_{\gro_i,\gro_j} =\xi+(1-\xi)\wrs_{\gro_i\gro_j}.
\end{align*}
Such a dynamic, reminiscent of a lazy random walk on a suitably defined Markov chain, is also a key building block in \citep{barucca_detectability_2017} and leads to the following conditional likelihood:
\begin{align}\label{Barruca1 link dynamic}
P(\adj|\gro)&= \prod_{i< j}^{\no} \bigg[\wrs_{\gro_i\gro_j}^{\adj_{ij}^1} (1-\wrs_{\gro_i\gro_j})^{1-\adj_{ij}^1}\nonumber\\& \quad \times \prod_{t=2}^T\Big[\xi\delta_{\adj_{ij}^t, \adj_{ij}^{t-1}}+ (1-\xi)\wrs_{\gro_i\gro_j}^{\adj_{ij}^t}(1-\wrs_{\gro_i\gro_j})^{1-\adj_{ij}^t}\Big]\bigg],
\end{align}
where $\delta$ is the usual Kronecker delta function.
\end{example}

In general, inferring the latent community memberships and the parameters in models such as the one we specified tend to be computationally difficult because of the combinatorial nature of the community membership estimation. While traditional approaches (e.g. \citet{xu_dynamic_2014,han_consistent_2015, zhang_random_2016,barucca_detectability_2017}) directly devise EM based methods to solve for the memberships and parameters simultaneously, often using greedy/local search heuristics to deal with the combinatorial aspect, we take a different approach: we first solve several marginalized problems efficiently (because we show that they follow the SBM structure for single graphs) and collate these memberships appropriately. Then given the estimated community memberships, we solve for the remaining parameters using maximum likelihood estimation in a straightforward manner. This last step is similar to the estimation of the parameters of multi-graph SBM with blocks specified a priori, and was first discussed in \citet{holland1983stochastic}.

%% file: sec_group_fixed_algo.tex
\subsection{Algorithm}\label{sec:group-fixed-algo} 

The outline of our algorithm is as follows: (a) we estimate the community memberships of nodes at each time index $t=1, ..., T$ by using any method designed for the SBM (for e.g., spectral clustering or a more sophisticated one in~\citet{gao_achieving_2015}), (b) we unify these estimated memberships, taking into account permutations and possible errors, to get a single estimate of community memberships, and (c) we obtain the estimates of rest of the parameters by maximizing likelihood in Equation (\ref{MLE_1}).\\ 

\noindent\textbf{Independent Community Membership Estimates}: Given $T$ adjacency matrices $\adj^t$ for $1 \leq t\leq T$ and the number of communities $\noc$, we obtain $T$ different community assignment estimates $\hat{\gro}^t$ for $1 \leq t\leq T$. Let \textsc{CMRecover} denote this step. This is a valid operation as the marginal likelihood at each time step resembles the likelihood of an SBM, as shown in the Lemma below.

\begin{lemma}\label{each_SBM}
Conditional on the community memberships, the marginal likelihood at each time step $t$ corresponds to the likelihood of an SBM with possibly different set of parameters.
\end{lemma}

\noindent\textbf{Unification}: We design a couple of algorithms to merge the $T$ estimated membership vectors. There are two key aspects to be considered here. First, even if there were no errors in each of the estimates, one estimate may only be equivalent to another up to a permutation (which can be denoted by a function $\per:[\noc]\rightarrow [\noc]$). Second, the SBM method used to estimate memberships may very well have some errors. Our algorithms below address both these issues. 

The first algorithm, called \textsc{UnifyCM} (see Algorithm~\ref{unify cms}), relies on taking an average of $T$ estimated cluster matrices $\hat{\y}^t$ (which are easily obtained given community membership estimates). That is, it keeps track of the frequency with which every pair of nodes are estimated to be in the same community across time, and then finds the closest cluster matrix to the average of the input cluster matrices. It assumes that the SBM routine makes errors in estimating the community memberships on a per node basis and that this error probability $\eps$ (i.e., the probability of assigning each node to an incorrect community) is known. If it is not known in practice, then one can set $\eps=0$ to run the algorithm.

\begin{algorithm}[h!]
\caption{\textsc{UnifyCM}}\label{unify cms}
\begin{algorithmic}[1]
\\\quad\textbf{Input:} Community assignments $\hat{\gro}^t$ for $1 \leq t\leq T$, estimated per-node error $\epsilon$
\\\quad\textbf{Output:} Unified community assignment $\hat{\gro}$ 
\\\quad Let $\hat{\y}^t:= (\hat{\gro}^t) \cdot(\hat{\gro}^t)^\prime \textrm{ for } \quad 1 \leq t\leq T$.
\\\quad $\hat{\y}:=\argmin_{\y \in \{\text{cluster matrices}\}} \norm{\y - \frac{1}{(1-\epsilon)^2 T}\sum_{t=1}^T \hat{\y^t}}_F$.
\\\quad Deduce $\hat{\gro}$ from $\hat{\y}$.
\end{algorithmic}
\end{algorithm}

When the SBM subroutine makes no errors ($\eps=0$) at all time steps (i.e. $\hat{\gro}^t$ is permutation-equivalent $ \gro^*$), Algorithm \ref{unify cms} outputs the correct community memberships:

\begin{lemma}\label{Unfiycm_correct}
\textsc{UnifyCM} (Algorithm \ref{unify cms}) outputs the correct community assignment vector $\gro^*$ if there is no error in the SBM subroutine at each time snapshot $t$.
\end{lemma}

For further analysis of \textsc{UnifyCM}, we make the following distributional assumption (\textbf{Assumption A}) on the community assignment errors by the SBM subroutine that provides inputs to \textsc{UnifyCM}.  Let the random variable $z_{ilt}$ represent community assignment of node $i$ to community $l$ at time $t$, which we assume is independent of other nodes and time. The random $\noc-$vector $z_{it}:=(\mathbbm{1}_{\hat{g}_i^t=l})_{l\in [\noc]}$, which aggregates these random variables is assumed to have the following distribution: 
\begin{align*}
l=\begin{cases}
\gro^*_{i}, \quad \textrm{with probability} \quad 1-\eps\\
m, \quad \textrm{with probability} \quad \frac{\eps}{\noc-1} \textrm{ for all } m \in [k]\setminus \{\gro^*_{i}\}.
\end{cases}
\end{align*}

The error in the output of \textsc{UnifyCM} can be stated in terms of the Frobenius norm difference between the true cluster matrix and the estimated cluster matrix. The following Theorem states that the probability of this error decays exponentially:

\begin{theorem}\label{UnifyCM_Grnt} Under \textbf{Assumption A}, the output of \textsc{UnifyCM} (Algorithm \ref{unify cms}) has the following guarantee:
\[P\Big(\norm{\hat{\y}-\y^*}_F>c_0 u-B\Big)\leq \no^2.\exp(-\frac{u^2}{c_a + c_b u}), \quad \forall u\geq 0,\]
where $c_0,c_a,c_b$ depend on $\eps,\noc,\no,\{k_l\}_{l \in [\noc]}$ and $T$. And $B$ is equal to $\min_{\y \in \{\text{cluster matrices}\}} \norm{\y - \frac{1}{(1-\epsilon)^2 T}\sum_{t=1}^T \hat{\y^t}}_F$.
\end{theorem}

To better understand the decrease in the probability of error of community assignments after running \textsc{UnifyCM}, we also state the probability of error of each individual output of the SBM subroutine:
\begin{lemma}\label{SSBM_Grnt} Under \textbf{Assumption A}, the output of the SBM subroutine at each time $t$ comes with the following guarantee:
\[P(\norm{\hat{\y}^t-\y^*}_F^2\geq \E[\norm{\hat{\y}^t-\y^*}_F^2]+u_1)\leq \exp(-c_2 u_1^2), \quad \forall u_1\geq 0.\]
\end{lemma}

We can see from the right hand sides of both the statistical guarantees above (Theorem \ref{UnifyCM_Grnt} and Lemma \ref{SSBM_Grnt}) that the error probabilities decay exponentially. Due to the averaging done by \textsc{UnifyCM}, the constants $c_a,c_b$ scale as $O(n^2)$ whereas $c_2$ scales as $O(n^3)$, giving much tighter concentration of the estimated communities to their true values in the former case. In other words, we achieve a lower probability of error at the output of our algorithm compared to its inputs. Another closely related algorithm \textsc{ConvexUnify}, which solves a convex relaxation of the combinatorial problem over cluster matrices to find community memberships, is discussed in the Appendix.

Our second algorithm (\textsc{UnifyLP}) focuses on casting the aggregation problem as a mathematical program that minimizes the error between membership estimates by searching over all permutations that can relate each estimate to the others simultaneously. Let matrix $Q^t \in \{0,1\}^{\no \times \noc}$ denote the estimated community assignments for the $\no$ nodes at time $t$. Each row $i$ of matrix $Q^t$ has only one entry equal to 1, denoting the community that node $i$ has been assigned to. Let $\tau^{t}$ be the $\noc \times \noc$ permutation matrix that maps $Q^{t}$ to the unknown optimal community assignment $Q^{final}$. Then, we can compute $\tau^t$ and $Q^{final}$ by solving the following problem:
\begin{align}\label{Nconvexoptassign}
&\min_{\tau^t , Q^{final}} \sum_{t=1}^T \lVert Q^t \tau^t-Q^{final}\rVert^2_F \;\;\textrm { s.t.}\\
&\sum_j Q^{final}_{i,j}=1 \quad \forall i \in [\no];\nonumber\\
&\sum_i \tau^t_{ij}=1 \quad \forall j\in [\noc], \; t\in[T];\nonumber\\
&\sum_j \tau^t_{ij}=1 \quad \forall i\in [\noc], \; t\in[T]; \nonumber\\
&\tau^t_{ij} \in \{0,1\} \textrm{ and } Q^{final}_{i,j} \in \{0,1\} \quad \forall i\in [\noc],\; j\in [\noc], \; t\in[T].  \nonumber
\end{align}
Problem \ref{Nconvexoptassign} is a hard combinatorial problem, and is
related to the consensus clustering problem in the data mining literature (and has a variety of heuristic solution approaches). Here, we also propose a principled (linear) approximation  that is easier to solve (it turns out to be an instance of the maximum-weight bipartite matching problem). The approximation is motivated in two ways. First, since we do not know which $Q^t$ has the least errors, we could first un-permute all other estimated $Q^t$’s with respect to an arbitrarily chosen one, and then compute community memberships based on a majority vote. The second motivation comes from the literature on rank aggregation, where people have approximated a similar initial combinatorial optimization problem. To be specific, our approximation is the following: we set the variable $Q^{final} = Q^{1}$ (w.l.o.g.). A consequence of this choice is that we can compute all the $\tau^t$ matrices (there is no need to compute $\tau^1$) by solving the following problem:

\begin{align}\label{optassign}
&\min_{\forall t \in [T]: \tau^t \in \mathbb{R}^{\noc \times \noc}} \sum_{t=2}^T \lVert Q^t \tau^t-Q^{1}\rVert^2_F \;\;\textrm { s.t.}\\
&\sum_i \tau^t_{ij}=1 \quad \forall j\in [\noc], \; t\in [T];\nonumber\\
&\sum_j \tau^t_{ij}=1 \quad \forall i\in [\noc], \; t\in [T]; \textrm{ and}\nonumber\\
&\tau^t_{ij} \in \{0,1\} \quad \forall i\in [\noc],\; j\in [\noc], \; t\in[T]. \nonumber
\end{align}
It can be shown that one can reduce the problem in Equation (\ref{optassign}) to $T-1$ separable problems (this is easily inferred), each of which has a linear objective function as shown below for any $t > 1$ (we drop the index $t$ from $\tau^t$ for simplicity here):
\begin{align*}
&\max_{\tau \in \mathbb{R}^{\noc \times \noc}} \sum_i\sum_j \tau_{ij} [{Q^{t}}^\prime Q^{1}]_{ij} \;\;\textrm { s.t.}\\
&\sum_i \tau_{ij}=1 \quad \forall j\in [\noc];\nonumber\\
&\sum_j \tau_{ij}=1 \quad \forall i\in [\noc]; \textrm{ and}\nonumber\\
&\tau_{ij} \in \{0,1\} \quad \forall i\in [\noc],\; j\in [\noc]. \nonumber
\end{align*} 
The problem is thus a max-weight bipartite matching problem and can be easily solved using linear programming relaxation or specialized combinatorial methods (such as the Hungarian algorithm). We solve $T-1$ such problems to get a unified community membership estimate $\hat{g}$, as shown in Algorithm~\ref{unify LP}. It is easy to show the following correctness property for this algorithm.
\begin{lemma}
If there exists a permutation $\per_t:[\noc]\rightarrow [\noc]$ such that $\per_t(\hat{\gro}^t) = \gro^*$ for all $t \in [T]$, then the output $\hat{\gro}$ of  \textsc{UnifyLP} is $g^*$. \label{claim:aac}
\end{lemma}

A brief intuition for this algorithm is as follows: as an example, consider the case when $Q^1$ has an error and all other estimates have no errors. Then, due to the majority voting done in line $7$ (of Algorithm~\ref{unify LP}), this error will get corrected.

\begin{algorithm}
\caption{\textsc{UnifyLP}}\label{unify LP}
\begin{algorithmic}[1]
\\\quad\textbf{Input:} Community assignments $\hat{\gro}^t$ for $t\in [T]$
\\\quad\textbf{Output:} Unified community assignment $\hat{\gro}$ 
\\\quad \textbf{for} $t=1 ,..., T$ form matrix $Q^t$ based on $\hat{\gro}^t$. 
\\\quad\textbf{for} $t=2 ,..., T$ solve Problem (\ref{optassign}).
\\\quad\textbf{for} $i=1 ,..., \no$
\\\quad\quad\quad $\hat{\gro}_{i}=\argmax_{k \in [\noc]}\sum_{t \in [T]}e_i\trn{Q^t\tau^t}$ \Comment{Here, $e_i$ is an $\no-$vector with $1$ in $i^{th}$-entry}
\end{algorithmic}
\end{algorithm}

\noindent\textbf{Maximum Likelihood Estimation and Marginalization}:
At this point, given that we have estimated community memberships, we can maximize the conditional likelihood of the observed graph sequence in Equation (\ref{MLE_1}) over the remaining parameters. Such an approach was first proposed by \citet{holland1983stochastic}. Algorithm \ref{Inference general} describes the previous community membership estimation step and this general MLE step. Below, we note special cases where the SBM structure at each time step can be further utilized in estimating parameters directly, via solving a system of equations.

\begin{algorithm}[htpb]
\caption{Inference of Community Memberships and Model Parameters}\label{Inference general}
\begin{algorithmic}[1]
\\\quad\textbf{Input:} $T$ adjacency matrices $\adj^1, \adj^2, ..., \adj^T$, number of communities $k$
\\\quad\textbf{Output:} Estimates $\hat{g},\wrs^t_{rs}, x$ for $1 \leq r ,s\leq \noc, \quad t\in [T]$
\\\quad \textbf{for} t=1,...,T $\quad \hat{g}^t$ = \textsc{CMRecover}($\adj^t, \noc$)
\\\quad  $\hat{g} $ = \textsc{UnifyCM}($\hat{g}^1, ..., \hat{g}^T$) or \textsc{UnifyLP}($\hat{g}^1, ..., \hat{g}^T$)
\\\quad \textbf{for} $1 \leq r,s \leq \noc,$ maximize log likelihood function $\log P(\adj|\gro)$ in Equation (\ref{MLE_1}) 
\\\quad \textbf{return} $\hat{g}; \hat{\wrs}_{rs}^t$ and $\;\hat{x} $ for $ 1 \leq r,s \leq \noc, \quad t \in [T]$.  
\end{algorithmic}
\end{algorithm}

\textbf{Alternative to MLE:} If $\wrs^t=\wrs$ for all $t$, we can reuse a key property that our model satisfies to estimate parameters instead, which is that the marginal distribution of the process governing the Markovian dynamics at any time step resembles the distribution of an SBM model. In particular, its parameters are related to the model parameters through a system of equations and we can solve for the model parameters using a strategy akin to using the method of moments (MoM) instead. Note that since in this case $\wrs^t=\wrs$, we have: $\pzo^t_{\gro_i\gro_j}=\pzo_{\gro_i\gro_j}$, and $\poo^t_{\gro_i\gro_j}=\poo_{\gro_i\gro_j}$ for all $1 \leq t\leq T$. Let $\wt^t_{g_ig_j}:=P(\adj^t_{ij}=1|\gro_i,\gro_j)$ denote the conditional marginal probability of a link between two nodes at time $t$ (conditioned on the community memberships). This probability can be interpreted as the link formation probability of a new SBM instance (we show this below). In other words, this new instance would have the same block structure (community memberships) as the model in Equation (\ref{MLE_1}), and its parameters would be related to the model parameters $\left(x,\wrs\right)$ in Equation (\ref{MLE_1}) as shown below:
\begin{align} \label{margin 1}
\wt^t_{g_ig_j}&=\sum_{\adj^{t-1}_{ij}} P(\adj^t_{ij}=1, \adj^{t-1}_{ij}|\gro_i,\gro_j)= p_{1\rightarrow 1}\wt^{t-1}_{\gro_i\gro_j}+p_{0\rightarrow 1}(1-\wt^{t-1}_{\gro_i\gro_j}), \nonumber
\\
&=\poo_{\gro_i\gro_j}\wt^{t-1}_{\gro_i\gro_j}+\pzo_{\gro_i\gro_j}(1-\wt^{t-1}_{\gro_i\gro_j})=(\poo_{\gro_i,\gro_j}-\pzo_{\gro_i,\gro_j})^{t-1}\wt^1_{\gro_i,\gro_j}+\sum_{l=0}^{t-2}(\poo_{\gro_i,\gro_j}-\pzo_{\gro_i,\gro_j})^l \pzo_{\gro_i,\gro_j},\nonumber\\
&=(\poo_{\gro_i,\gro_j}-\pzo_{\gro_i,\gro_j})^{t-1}\wrs_{\gro_i\gro_j}+\frac{1-(\poo_{\gro_i,\gro_j}-\pzo_{\gro_i,\gro_j})^{t-1}}{1-(\poo_{\gro_i,\gro_j}-\pzo_{\gro_i,\gro_j})}\pzo_{\gro_i,\gro_j},
\end{align}

where $\wt^1=\wrs$ by definition. For the model in Example~\ref{fixed-model-eg1} (type-I dynamic),  Equation (\ref{margin 1}) can be specialized as:
\begin{align}\label{Bernouli recursion}
\wt^t_{\gro_i,\gro_j} &= (1-\mu_{\gro_i,\gro_j}-\mu_{\gro_i,\gro_j}\wrs_{\gro_i,\gro_j})^{t-1} \wrs_{\gro_i,\gro_j}+ \mu_{\gro_i,\gro_j}\wrs_{\gro_i,\gro_j} \frac{1-(1-\mu_{\gro_i,\gro_j}-\mu_{\gro_i,\gro_j}\wrs_{\gro_i,\gro_j})^{t-1}}{\mu_{\gro_i,\gro_j}+\mu_{\gro_i,\gro_j}\wrs_{\gro_i,\gro_j}},\nonumber\\
&= \frac
{(1-\mu_{\gro_i,\gro_j}-\mu_{\gro_i,\gro_j}\wrs_{\gro_i,\gro_j})^{t-1}\wrs_{\gro_i,\gro_j}^2+\wrs_{\gro_i,\gro_j}}{1+\wrs_{\gro_i,\gro_j}}.
\end{align} 
The above relates the parameter of the new SBM instance (namely the one related to the $t^{th}$ graph snapshot) to the original parameters of the model in Example~\ref{fixed-model-eg1}. Similarly, for the model in Example~\ref{fixed-model-eg2} (type-II dynamic), the corresponding relation is:
\begin{align}\label{Barrucca W recursion}
\wt^t_{\gro_i,\gro_j} =\xi^{t-1}\wrs_{\gro_i,\gro_j}+\wrs_{\gro_i,\gro_j}(1-\xi^{t-1})=\wrs_{\gro_i,\gro_j}.
\end{align}

As for the parameters $\wt^t$ themselves, we can estimate these marginal probabilities at each time step using straightforward and much simpler maximum likelihood estimation (which amounts to counting and normalizing). For instance, for the models in Examples~\ref{fixed-model-eg1} and ~\ref{fixed-model-eg2}, the initial time step probability only depends on $\wrs_{rs}$, where $r$ and $s$ are community indices for nodes $i$ and $j$. That is, $P(\adj^{1}_{ij}|\wrs_{rs})=\wrs^{\adj_{ij}^{1}}_{rs}(1-\wrs_{rs})^{1-\adj_{ij}^{1}}$ with $g_i=r$ and $g_j=s$. The maximum likelihood estimation given the whole graph is simply $\hat{\wt}^1_{rs}=\frac{\sum_{\{i,j: i\in r, j\in s\}}\adj_{ij}^{1}}{|r||s|}$ for $r\neq s$ (with a slight correction when $r=s$). In our final step, we plug these estimates $(\hat{\wt}^t)$ into the marginal probability equations derived above, and solve for the model parameters that minimize the error between the estimates and functions of our dynamic model parameters. All these steps are reproduced in Algorithm~\ref{Inference general Fixed model} for the model in Example~\ref{fixed-model-eg1} and in Algorithm~\ref{Inference Fixed Barrucca} for the model in Example~\ref{fixed-model-eg2} for better readability.

\begin{algorithm}[htpb]
\caption{Inferring Parameters of Model in Example~\ref{fixed-model-eg1}}\label{Inference general Fixed model}
\begin{algorithmic}[1]
\\\quad\textbf{Input:} $T$ adjacency matrices $\adj^1, \adj^2, ..., \adj^T$, number of communities $k$
\\\quad\textbf{Output:} Estimates $\hat{g},\wrs_{rs}, \mu_{rs}$ for $1 \leq r ,s\leq \noc$
\\\quad \textbf{for} t=1,...,T $\quad \hat{g}^t$ = \textsc{CMRecover}($\adj^t, \noc$)
\\\quad  $\hat{g} $ = \textsc{UnifyCM}($\hat{g}^1, ..., \hat{g}^T$) or \textsc{UnifyLP}($\hat{g}^1, ..., \hat{g}^T$)
\\\quad \textbf{for} $1 \leq r,s \leq \noc, 1 \leq t \leq T$ $\quad \hat{\wt}^t_{rs}=\frac{\sum_{\{i,j: i\in r, j\in s\}}\adj_{ij}^{t}}{|r||s|}$ and $\hat{\wt}^t_{rr}=\frac{\sum_{\{i,j: i\in r, j\in r\}}\adj_{ij}^{t}}{|r||r-1|}$
\\\quad \textbf{for} $1 \leq r,s \leq \noc,$ $1\leq t\leq T$ \quad find $({\wrs}_{rs,t}^*, {\mu}_{rs,t}^*) =\argmin_{\wrs_{rs},\mu_{rs}} ({\wt}_{rs}^t-\hat{\wt}_{rs}^t)^2$ from Equation (\ref{Bernouli recursion}) (for instance, by grid search on  $[0,1]\times[0,1]$)
\\\quad \textbf{return} $\hat{g}; \frac{1}{T}\sum_{t=1}^T\hat{\wrs}_{rs,t}^*$ and $\; \frac{1}{T}\sum_{t=1}^T \hat{\mu}_{rs,t}^* $ for $ 1 \leq r,s \leq \noc$  
\end{algorithmic}
\end{algorithm}

\begin{algorithm}[htpb]
\caption{Inferring Parameters of Model in Example~\ref{fixed-model-eg2}}\label{Inference Fixed Barrucca}
\begin{algorithmic}[1]
\\\quad\textbf{Input:} $T$ adjacency matrices $\adj^1, \adj^2, ..., \adj^T$, number of communities $k$
\\\quad\textbf{Output:} Estimates $\hat{g},\xi, \wrs_{rs}$ for $1 \leq r ,s\leq \noc$
\\\quad \textbf{for} t=1,...,T $\quad \hat{g}^t$ = \textsc{CMRecover}($\adj^t, \noc$)
\\\quad  $\hat{g} $ = \textsc{UnifyCM}($\hat{g}^1, ..., \hat{g}^T$) or \textsc{UnifyLP}($\hat{g}^1, ..., \hat{g}^T$)
\\\quad \textbf{for} $1 \leq r,s \leq \noc, 1 \leq t \leq T$ $\quad \hat{\wt}^t_{rs}=\frac{\sum_{\{i,j: i\in r, j\in s\}}\adj_{ij}^{t}}{|r||s|}$ and $\hat{\wt}^t_{rr}=\frac{\sum_{\{i,j: i\in r, j\in r\}}\adj_{ij}^{t}}{|r||r-1|}$
\\\quad \textbf{for} $1 \leq r,s \leq \noc \quad$ set $\hat{\wrs}_{rs}=\frac{1}{T}\sum_{t=1}^T \hat{\wt}_{rs}^t$
\\\quad Compute the maximum likelihood estimate $\hat{\xi}$ of $\xi$ by replacing $\hat{\wrs}_{rs}$, $1 \leq r,s \leq \noc$ in Equation (\ref{Barruca1 link dynamic}) (for instance, using grid search over $[0,1]$)
\\\quad \textbf{return} $\hat{g}, \hat{\xi}$ and 
$\hat{\wrs}_{rs}$ for $1 \leq r ,s\leq \noc$
\end{algorithmic}
\end{algorithm}

%% file: sec_group_changing.tex
\newpage
\section{Capturing Link and Community Persistence}\label{sec:group-changing-model}

\subsection{Model}

In the second model, nodes can belong to different communities as a function of time. Let the $T$ latent community membership vectors be $\gro=\{\gro^1, ..., \gro^T\}$. In the initial time step, let the community membership of each node be based on a $\noc$-dimensional prior vector $\{\pri_l\}_{l=1}^{\noc}$ over the $\noc$ choices. In fact, similar to literature (e.g. \citet{barucca_detectability_2017,ghasemian_detectability_2016}), we will assume that $\pri_l=\frac{1}{\noc}$ (the general case can also be addressed in our model as long as each community is not too small). Given community memberships at the initial time step, the edges in that time step are generated as follows: let $\wrs^1$ be the $\noc \times \noc$-dimensional parameter matrix, then nodes $i,j\ \in [\no]$ are connected to each other independently with probability $\wrs^1_{\gro^1_i,\gro^1_j}$. 

For the subsequent $T-1$ snapshots, we assume that the links between pair of nodes either are copied (persisted) from the previous time step with probability $\xi$ each, or they are generated again using $\wrs^t$ (possibly different from $\wrs^1$) and the current community memberships with probability $1-\xi$. We now describe how the community memberships change over time. In particular, we assume that a small minority of members in each community will change their membership from one time step to another \emph{arbitrarily}. At a given time step, the member nodes that do not change their communities in the next time step are denoted \emph{majority}. The rest of the member nodes are denoted \emph{minority}, and these change their community membership to other values uniformly at random, i.e., with probability $\frac{1}{\noc -1}$ similar to \citet{barucca_detectability_2017,ghasemian_detectability_2016} (this uniform choice can also be relaxed). We define a majority indicator variable (an auxiliary variable) for each node $i$ at each time step $t$, denoted $\maj_i^t$, as follows:
\[   
M_i^t = 
     \begin{cases}
       1, &\quad\text{if $\gro_i^t=\gro_i^{t+1}$}, \textrm{and}\\
       0 &\quad\text{if $\gro_i^t\neq\gro_i^{t+1}$}. \\ 
     \end{cases}
\]

At this point we could introduce a probabilistic model describing which node becomes a majority or a minority at each time step. This would complete the stochastic process for the persistence of community memberships and links. Instead, we assume access to exogenous information that specifies the majority/minority information of all nodes at time step $t$, and we denote it by $\maj^t$. Thus, instead of a stochastic model of evolution, we assume direct access to these auxiliary variables. The advantage of doing this is that we do not need to know the process by which nodes change their community memberships; it could be the result of a stochastic process or it could be a complex deterministic process (such as the outcomes of coordination games for instance). For example, consider the ex-customers of an Internet service provider (ISP), who do not extend their contract for the next year. It is difficult to determine which community (new ISP) they have transitioned to; however, it is easy to mark them as minority because they have changed their service provider. In this case the original ISP can estimate which new ISPs these customers transitioned into if they can still observe the user's network (say interaction, social etc). As another example, consider the student population of an elementary school. There may be some students who do not show up next year, which could easily imply that they changed their school. Yet another example comes from ecology. It is relatively easier to detect that some inhabitants have left their ecosystem and that these minority members do not exist in that region anymore, than to directly know where they may have migrated to. In all these three examples, an external observer may have access to coarse information ($M^t, \; t=1,...,T$) and not the actual community memberships (ISP, school, ecosystem). With this information and the general network information, they can estimate their likely communities. We note two points: (a) such metadata may not always be available in applications, and (b) in cases such as the above, it is more realistic and plausible for a firm or observer to known the majority/minority information than not knowing anything about the latent community memberships. We provide further justification for our modeling choices after writing down the likelihood functions for our model below.

Conditional on the knowledge of $M = (M^t)_{t=1,...,T-1}$, we now describe the dynamics of community membership vectors. For node $i$ at time $t$, $P(\gro_i^{t}|\gro_i^{t-1},\maj^{t-1}_i)=\maj_i^{t-1}\delta_{\gro_i^t,\gro_i^{t-1}}+\bar{\maj}_i^{t-1}\bar{\delta}_{\gro_i^t,\gro_i^{t-1}}\frac{1}{\noc-1}$, where $\bar{\maj}_i^{t-1}:=1-{\maj}_i^{t-1}$ and $\bar{\delta}_{\gro_i^t,\gro_i^{t-1}}=1-{\delta}_{\gro_i^t,\gro_i^{t-1}}$. Further, the joint probability of observing all $T$ community membership vectors is:
\begin{align}\label{eqn:dyn-model-changing-communities1}
P(\gro|\maj)=\prod_{i=1}^N\Big[\prod_{t=2}^T \maj_i^{t-1}\delta_{\gro_i^t,\gro_i^{t-1}}+\bar{\maj}_i^{t-1}\bar{\delta}_{\gro_i^t,\gro_i^{t-1}}\frac{1}{\noc-1}\Big]\frac{1}{\noc}.
\end{align}

The link persistence dynamics that we described above can be written as a conditional (on community memberships) likelihood as shown: 
\begin{align}\label{eqn:dyn-model-changing-communities2}
\hspace{-0.1in}P(\adj|\gro)=\quad &\prod_{i< j}^{\no} \wrs_{\gro_i^1\gro_j^1}^{1^{\adj_{ij}^1}} (1-\wrs^1_{\gro_i^1\gro_j^1})^{1-\adj_{ij}^1}\nonumber\\
&\times \prod_{t=2}^T\xi \delta_{\adj_{ij}^t, \adj_{ij}^{t-1}}+ (1-\xi)\wrs_{\gro_i^t\gro_j^t}^{t^{\adj_{ij}^t}}(1-\wrs_{\gro_i^t\gro_j^t})^{t^{1-\adj_{ij}^t}}.
\end{align}

\noindent\textbf{Further Motivation for Exogenous Majority/Minority Labels:}  At a first glance, the knowledge about which nodes change communities seems to be asking much more information compared to approaches that don't assume access to such data. Certainly, the information assumed is coarser than knowing the communities themselves (they are latent), otherwise the problem becomes trivial. But as we describe here, the level of information needed in our model is realistic and not difficult to obtain in practice.  In particular, there are three related questions which we address here: (a) does knowing this information make estimating community memberships easy?; (b) how strong is this assumption compared to not knowing such data?; and (c)  how can such information be obtained?

Knowing the coarser level information as assumed in our model (i.e., whether communities changed or not) still makes the community membership estimation problem non-trivial. This is because we do not assume any stochastic process driving these auxiliary variables, and this makes community estimation more difficult than the setting where one assumes a stochastic model on the latent community memberships directly. Nonetheless, we design algorithms that estimate the memberships of nodes that change communities one step into the future (see Theorems \ref{2SBM_thm} and \ref{General_Min_label}). Beyond that, due to our minimal assumptions, we are unable to estimate the memberships of these nodes. The SBM substructure that our methods exploit is completely lost. 

The alternative setting when no such information is observable is to consider the variables representing whether a node changed community at time $t$ as additional latent variables. In this case, a stochastic model needs to be specified and an EM based approach to infer these variables can potentially be devised. Unfortunately, such an approach (as far as we know) will most likely be intractable. It offers no benefit compared to the setting where we discard these additional latent variables and directly work on a stochastic latent variable model describing how communities themselves change (see the community changing models of \citet{ghasemian_detectability_2016,barucca_detectability_2017,xu_stochastic_2015,rastelli2017exact} for instance). Unfortunately in these latent variable models, it is not generally possible to reuse SBM machinery, and inference approaches typically have to fall back on EM (e.g., we empirically compare our methods with the DSBM approach proposed by \citet{xu_dynamic_2014} in Section~\ref{sec:experiments}), which again tend to be heuristic in nature and computationally inefficient. If we make very little assumptions on how communities change, then we essentially have a very large combinatorial problem at hand (e.g., as seen in DSBM~\citet{xu_dynamic_2014} and SBTM~\cite{xu_stochastic_2015}),  which is much larger than a single graph community estimation problem (DSBM and SBTM use local search heuristics). 

We note that there are parallels between our model and other learning settings such as the robust PCA problem, matrix completion problem etc. where similar metadata (or side information) is often considered available and algorithms are designed to use this information for efficient inference. Because we assume that the majority/minority knowledge is known and exogenous, we naturally have lesser modeling assumptions, and consequently our inference is more robust to model mis-specification.  These majority labels can be generated based on direct measurements (which was discussed earlier) as well as via prediction models. For example, these binary labels can be the target of an auxiliary classification problem when time varying features related to the nodes are also available.

%% file: sec_group_changing_algo.tex
\subsection{Algorithm}\label{sec:group-changing-algo}

Because we have specified a stochastic process describing the link persistence conditional on community memberships, which themselves are dictated by exogenous majority/minority information, we cannot completely recover all community memberships given the observed graph sequence without further assumptions. As a consequence, we propose two approaches, with different partial recovery goals. In the first, we discard nodes that have changed communities and reuse the methods developed in the previous section directly. In the second, we show that we can partially recover future community memberships of nodes that have changed their memberships recently.\\

\noindent\textbf{Reusing Algorithms from Section~\ref{sec:group-fixed-algo}}: In this straightforward approach, we remove minority nodes progressively and retain modified graphs with only those nodes which are majority, and apply the methods in Section~\ref{sec:group-fixed-algo}. The advantage of this approach is that it can work for a general choice of functions $\pzo^t_{\gro^t_i,\gro^t_j}(x,\wrs^t)$ and $\poo^t_{\gro^t_i,\gro^t_j}(x, \wrs^t)$ in Equation (\ref{MLE_1}) instead of the more restricted model of link persistence dynamics in Equation (\ref{eqn:dyn-model-changing-communities2}). Because these majority nodes don't change this community memberships, we can use Algorithm~\ref{Inference general} with minor modifications, as shown in Algorithm \ref{remove algo} below.

\begin{algorithm}
\caption{Inferring Parameters by Removing Minorities}\label{remove algo}
\begin{algorithmic}[1]
\\\quad\textbf{Input:} $T$ adjacency matrices $\adj^1, \adj^2, ..., \adj^T$, vectors $\maj^1, \maj^2, ..., \maj^{T-1}$, and the number of communities $k$.
\\\quad\textbf{Output:} Estimates $\hat{\gro}, \wrs^t_{rs}, x$ for $1 \leq r ,s\leq \noc, \quad t\in [T]$.
\\\quad \textbf{for} t=1,...,T \\
\quad \quad Remove nodes for which $\maj^{t-1}$ entry is $0$ for some $t$m and update $\adj^1, ..., \adj^T$.\\
\quad Run Algorithm~\ref{Inference general} with inputs $ \adj^1, ..., \adj^T$ to obtain estimates of $x$ and $\wrs^t\;\forall t \in [T]$.\\
\quad \textbf{return} $\hat{\gro}$ and estimates $\hat{\wrs^t},\hat{x}\; \forall t\in [T]$.
\end{algorithmic}
\end{algorithm}

The advantage of being able to reuse \textsc{CMRecover}, i.e., fast methods developed for the single graph SBM model, is retained in this setting. Quite noticeably, the main drawback here is that the information about majority/minority is only used as a pre-preprocessing step to remove parts of the graphs that we know do not conform to the Markovian dynamics needed for Algorithm~\ref{Inference general}. As a consequence, we are unable to estimate community memberships of any node once it becomes a minority. We now address this issue by being able to partially estimate memberships of nodes that are not majority in a reliable way. \\

\noindent\textbf{Exploiting Block Sub-structure to Estimate Additional Community Memberships}: 
To describe this approach, we assume that all nodes do not change communities till time $t-1$ (this is only for exposition, otherwise we can retain only those nodes that haven't changed communities up to time $t-1$ and proceed). Then, in addition to estimating community memberships of majority nodes, we can also output the estimated memberships at time $t$ of nodes that are minority at time index $t-1$, which was not possible in the previous approach.

Consider the graphs at time $t-1$ and $t$. Because all nodes have the same community memberships across time till time $t-1$, we can use \textsc{CMRecover} and \textsc{UnifyCM}/\textsc{UnifyLP} to get estimated community memberships as before. Since $\maj^{t-1}$ is given, we know which nodes are majority or minority at the current time step. Because of the conditional Markovian dynamics in this model, we can show that at time $t$, the majority nodes (at time $t-1$) follow a block structure, and the minority nodes (at time $t-1$) also follow a different block structure. Due to this property, we can run \textsc{CMRecover} on these two  sets of nodes separately at time $t$, and then relate the communities recovered across the two sets. This relation is possible due to different probabilities with which edges are formed between majority nodes, between minority nodes and across majority/minority nodes. In this way, we can get community membership estimates at time $t$ for more nodes (in particular minority nodes at time $t-1$) than the previous simpler approach.

We now formally prove the properties claimed above. The following theorem shows how the marginal probability of an edge at time $t$ (when nodes have changed communities) is related to the parameters of the dynamic model.
\begin{theorem}\label{2SBM_thm}
Assuming that nodes have not changed communities till time index $t-1$, $\wrs_{ij}^t(a,b) := P(\adj_{ij}^t=1|\gro_i^t=a, \gro_j^t=b, \maj)$ is related to the parameters of the dynamic model $(\xi,\wrs^t)$ in Equations (\ref{eqn:dyn-model-changing-communities1}) and (\ref{eqn:dyn-model-changing-communities2}) according to the following:
\begin{align}\label{2_SBM}
\wrs_{ij}^t(a,b)&= \xi \Bigg[\Big[\maj_i^{t-1}\maj_j^{t-1}\wrs_{ab}^{t-1}\Big]+\Big[\frac{\maj_i^{t-1}\bar{\maj}_j^{t-1}}{\noc-1} \sum_{\gro_j^{t-1}\neq b} \wrs_{a\gro_j^{t-1}}^{t-1}\Big]\\
&+\Big[\frac{\bar{\maj}_i^{t-1}\maj_j^{t-1}}{\noc-1}\sum_{\gro_i^{t-1}\neq a}\wrs_{\gro_i^{t-1}b}^{t-1}\Big]+\Big[\frac{\bar{\maj}_i^{t-1}\bar{\maj}_j^{t-1}}{(\noc-1)^2}\sum_{\gro_i^{t-1}\neq a,\gro_j^{t-1}\neq b}\wrs_{\gro_i^{t-1}\gro_j^{t-1}}^{t-1}\Big] \Bigg]\nonumber\\
&+ (1-\xi) \wrs^t_{ab}.\nonumber
\end{align}
\end{theorem} 

From Equation (\ref{2_SBM}), we observe that if we know the community memberships of $i,j$ at time $t-1$ and both of them are either majority or minority, then $\wrs^t_{ij}(a,b)$ can be viewed as probabilities corresponding to a new SBM model each. That is, both the majority set and the minority set of nodes can be modeled using two auxiliary SBM models at time $t$. The following theorem establishes that the community memberships obtained separately for the majority and minority nodes at time $t$ can be aligned to each other, thus enabling us to estimate the membership of the minority nodes (i.e., minority at time $t-1$) at time $t$ as well.

\begin{theorem}\label{General_Min_label}
Community memberships of nodes in the minority set at time $t-1$ can be computed at time $t$ if $\wrs_{ij}^t(a,a)>\wrs_{ij}^t(a,b) \quad \forall b\neq a$ This also holds if $\wrs_{ij}^t(a,a)<\wrs_{ij}^t(a,b) \quad \forall b\neq a$. 
\end{theorem}

Given the result above, we can relate the community memberships obtained at time $t$ for majority nodes with the community memberships obtained at time $t$ for the minority nodes, and thus get the estimates of community memberships for minority nodes. We can specialize our model to mimic that of \citet{barucca_detectability_2017} (albeit with lesser assumptions on our end), and in this setting the condition for recovery given in Theorem \ref{General_Min_label} always holds. We discuss this special case next.\\

\textbf{Special Case:} We assume that $\wrs^t=\wrs$ has only two degrees of freedom: a constant $\wrs_{in}$ on the diagonal and another constant $\wrs_{out}\leq \wrs_{in}$ on the off diagonals (such an assumption has been used widely before, including in ~\cite{barucca_detectability_2017,ghasemian_detectability_2016,xu_optimal_2017,gao_achieving_2015}). As noted in these works, instead of working with $\wrs_{in}$ and $\wrs_{out}$, it is easier to work with a different parameter $\asor \in [0,1]$, which measures the \emph{assortativity} of $\wrs$. In particular, we let
\begin{align}\label{assrotative_model_eq}
\wrs=\asor \noc \bar{\wrs} \mathbb{I} + (1-\asor) \bar{\wrs}\textbf{1}\textbf{1}'.
\end{align}
Here, matrix $\wrs$ is essentially a convex combination  between two modes: being fully assortative ($\noc \bar{\wrs} \mathbb{I}$) and being fully random $\bar{\wrs}\textbf{1}$, where $\bar{\wrs}=\frac{1}{\noc^2} \sum_{rs} \wrs_{rs}$ for $r,s \in [\noc]$ is the mean of all entries, $\mathbb{I}$ is the identity matrix and $\textbf{1}$ is the all ones vector.\\

To formally show that the majorities and the minorities retain block structures at time $t$ and that community memberships of minorities can be estimated at time $t$, we start with a few supporting results that are needed in the proof of Corollary~\ref{prop baruca marginal}.

\begin{lemma} \label{claim baruca marginal} 
Assuming that nodes have not changed communities till time index $t-1$, the following properties hold at time $t$: (1) $P(\gro_i^t|\maj)=\frac{1}{\noc}$, (2) $P(\adj_{ij}^t=1|\maj)=\bar{\wrs}$, and (3) $\frac{1}{\noc}\sum_{a} P(\adj_{ij}^t=1|\gro_i^t=a, \gro_j^t=b, \maj) =\bar{\wrs} \quad \forall b \in [\noc]$.
\end{lemma}

Given these three properties, similar to Theorem \ref{2SBM_thm}, Corollary~\ref{prop baruca marginal} below also relates the marginal probability of an edge at time $t$ (when some nodes have changed communities) to the parameters $(\xi,\alpha,\bar{W})$ of the dynamic model.
\begin{corollary}\label{prop baruca marginal} Assuming that nodes have not changed communities till time index $t-1$, $\wrs_{ij}^t(a,b) := P(\adj_{ij}^t=1|\gro_i^t=a, \gro_j^t=b, \maj)$ is related to the parameters of the dynamic model in Equations (\ref{eqn:dyn-model-changing-communities1}) and (\ref{eqn:dyn-model-changing-communities2}) according to the following:
\begin{align}\label{marginal barrucca}\wrs_{ij}^t(a,b)&= \xi \Bigg[\maj_i^{t-1} \maj_j^{t-1} \wrs_{ab}+ \frac{\maj_i^{t-1}\bar{\maj}_j^{t-1}}{\noc-1}(\noc \bar{\wrs}-\wrs_{ab})+\frac{\bar{\maj}_i^{t-1}{\maj}_j^{t-1}}{\noc-1} \nonumber\\&\times(\noc\bar{\wrs}-\wrs_{ab})+\frac{\bar{\maj}_i^{t-1}\bar{\maj}_j^{t-1}}{(\noc-1)^2}\Big((\noc^2-2\noc)\bar{\wrs}+\wrs_{ab}\Big) \Bigg]+ (1-\xi) \wrs_{ab}.
\end{align}
Equivalently, 
\begin{align}\label{recursion}
\wrs_{ij}^t(a,b)= \xi(\gamma^{t-1}+ \psi^{t-1}\wrs_{ab})+(1-\xi)\wrs_{ab},
\end{align} 
where $\gamma^{t-1}=\frac{\maj_i^{t-1} \bar{\maj}_j^{t-1}}{\noc-1}\noc\bar{\wrs}+\frac{\bar{\maj}_i^{t-1} \maj_j^{t-1}}{\noc-1}\noc \bar{\wrs}+\frac{\bar{\maj}_i^{t-1} \bar{\maj}_j^{t-1} (\noc^2-2\noc)}{(\noc-1)^2}\bar{\wrs}$ and $\psi^{t-1}=\maj_i^{t-1}\maj_j^{t-1}-\frac{\maj_i^{t-1}\bar{\maj}_j^{t-1}}{\noc-1}-\frac{\bar{\maj}_i^{t-1} \maj_j^{t-1}}{\noc-1}+\frac{\bar{\maj}_i^{t-1} \bar{\maj}_j^{t-1}}{(\noc-1)^2} $.
\end{corollary}

Again, Equation (\ref{marginal barrucca}) shows that SBM structures are retained for the majority set and the minority set of nodes at time $t$ (that is, the edge formation probability depends on the local latent community memberships of the nodes involved). The following corollary states that the community memberships of the majority and minority nodes at time $t$ can be aligned to each other without requiring any further condition, which was the case with Theorem~\ref{General_Min_label}.

\begin{corollary}\label{claim_Min_label} Under the model described by Equation~(\ref{assrotative_model_eq}), the community memberships of all nodes in the minority set at time $t-1$ can be computed at time $t$. 
\end{corollary}

Thus for the special case, given the result in Corollary~\ref{claim_Min_label}, we can relate the community memberships obtained at time $t$ for majority nodes with the community memberships obtained at time $t$ for the minority nodes, and thus get the estimates of community memberships for minority nodes that was not possible before. Estimation of parameters $\xi$, $\asor$ and $\bar{\wrs}$ is similar to the previous settings, and the complete algorithm is presented in Algorithm~\ref{Inference 2nd Model}.

\begin{algorithm}
\caption{Inferring Parameters by Exploiting Block Sub-structure}\label{Inference 2nd Model}
\begin{algorithmic}[1]
\\\quad\textbf{Input:} Adjacency matrices $\adj^1,\adj^2,...,\adj^{t-1}, \adj^{t}$, number of communities $k$, and $\maj$. Here $t-1$ is the time index before which all nodes are majority.
\\\quad\textbf{Output:} Community memberships up to time $t$, and estimates $\xi, \asor, \bar{\wrs}$
\\\quad Form $\adj^{maj}_{t}$ ($\adj^{min}_{t}$) by removing minority (majority) nodes and their links
\\\quad Obtain $\hat{g}^{t-1},\hat{\xi},\hat{W}_{rs}$ for $1 \leq r,s \leq k$ by apply Algorithm~\ref{Inference Fixed Barrucca} on ($\adj^{1},...,\adj^{t-1}$).\\
\quad $\hat{g}^{t}_{maj}$ = \textsc{CMRecover}($\adj^{maj}_{t}, \noc$).\\
\quad $\hat{g}^{t}_{min}$ = \textsc{CMRecover}($\adj^{min}_{t}, \noc$).\\
\quad Assign each minority community at ${t}$ with that majority community at ${t}$ that has the highest (if $1-\hat{\xi}-\frac{\hat{\xi}}{k-1} > 0$, or alternatively the lowest, if $1-\hat{\xi}-\frac{\hat{\xi}}{k-1} \leq 0$) number of links with the former, to get $\hat{g}^{t}$ for all nodes.\\
\quad Estimate $\hat{\asor}, \hat{\overline{\wrs}}$ from the set $\{\hat{\wrs}_{rs} \;\forall r,s \in [k]\times[k]$ by minimizing mean squared error.\\
\quad \textbf{return} $\hat{g}^{t-1}, \hat{g}^{t},\hat{\asor}, \hat{\overline{\wrs}}$ and $ \hat{\xi}$.
\end{algorithmic}
\end{algorithm}

%% file: sec_experiments.tex
\section{Experiments}\label{sec:experiments}

In this section, we perform several experiments to complement the analysis of the proposed methods in Sections~\ref{sec:group-fixed-model} and~\ref{sec:group-changing-model}. The code supporting the experiments is available online\footnote{ \url{https://github.com/thejat/dynamic-network-growth-models}}. The first set of experiments, which look at recovery of community estimates on synthetic instances, have the following goals:
\begin{itemize}
\item Assess the error correction capabilities of UnifyCM and UnifyLP (Experiment A).
\item Understand the regimes in which UnifyCM and UnifyLP fare better than Spectral-Mean and vice-versa (Experiment B).
\item Quantify the quality of estimation of minority community memberships by Algorithm~\ref{Inference 2nd Model} when both links and communities persist across time (Experiment C).
\end{itemize}
The second set of experiments (Experiment D and E) look at both parameter estimation and latent community estimation in two cases: (a) when the edge probability matrix is fixed across time, and (b) the more general setting when the edge probability matrices are drawn i.i.d. from an unknown distribution. 

In the final set of experiments (Experiments F and G), we compare the community recovery performance of our methods in comparison to a more general method (given data generated from the latter) and also show the link prediction accuracy on the Enron email and Facebook friendship datasets.

In all the experiments, we choose spectral clustering (from the python Scikit-learn package\footnote{\url{http://scikit-learn.org/stable/index.html}}) as the single graph SBM sub-routine whenever necessary. For MLE estimation, we used the optimize function in the python Scipy optimize library\footnote{\url{https://docs.scipy.org/doc/scipy/reference/optimize.html}} (in particular, using gradient information in a truncated Newton algorithm). When comparing estimated community memberships with the ground truth, we use $1 - $ NMI (normalized mutual information) to report performance (lower is better) and when comparing estimated matrices to true parameter matrices, we use the normalized Frobenius norm error (normalized by the Frobenius norm of the true parameter matrix). For comparing scalar estimates, we simply use the relative error (absolute value difference divided by the magnitude of the true value). Unless otherwise stated, all numbers reported are averaged over 10 Monte Carlo runs. We have omitted variance around the error metrics if they are relatively small.


In the first experiment (Experiment A), we evaluate the error correction capabilities of both UnifyCM (Algorithm~\ref{unify cms}) and UnifyLP (Algorithm~\ref{unify LP}). For both these algorithms, the input is a sequence of estimated community memberships which are output by a generic SBM sub-routine (spectral clustering). We plot the average of errors at the input (averaged over the length of the graph sequence) and the output as shown in Figure~\ref{fig:ip_vs_op}. The error is measured in terms of the normalized mutual information (NMI) metric (in particular we use 1- NMI). These numbers are further averaged over 10 Monte Carlo runs. For each of the three configurations, viz., $(n,k) = (100,4)$ with type-I dynamic, $(n,k) = (100,4)$ with type-II dynamic, and $(n,k) = (500,2)$ with type-I dynamic, we can observe in the plots that the output errors are greatly reduced as compared to that of the inputs. A similar trend is observed for both UnifyCM and UnifyLP. Further, as expected, these errors decrease over time at the output for both algorithms.

\begin{figure}	
\centering
 \includegraphics[width=0.3\columnwidth]{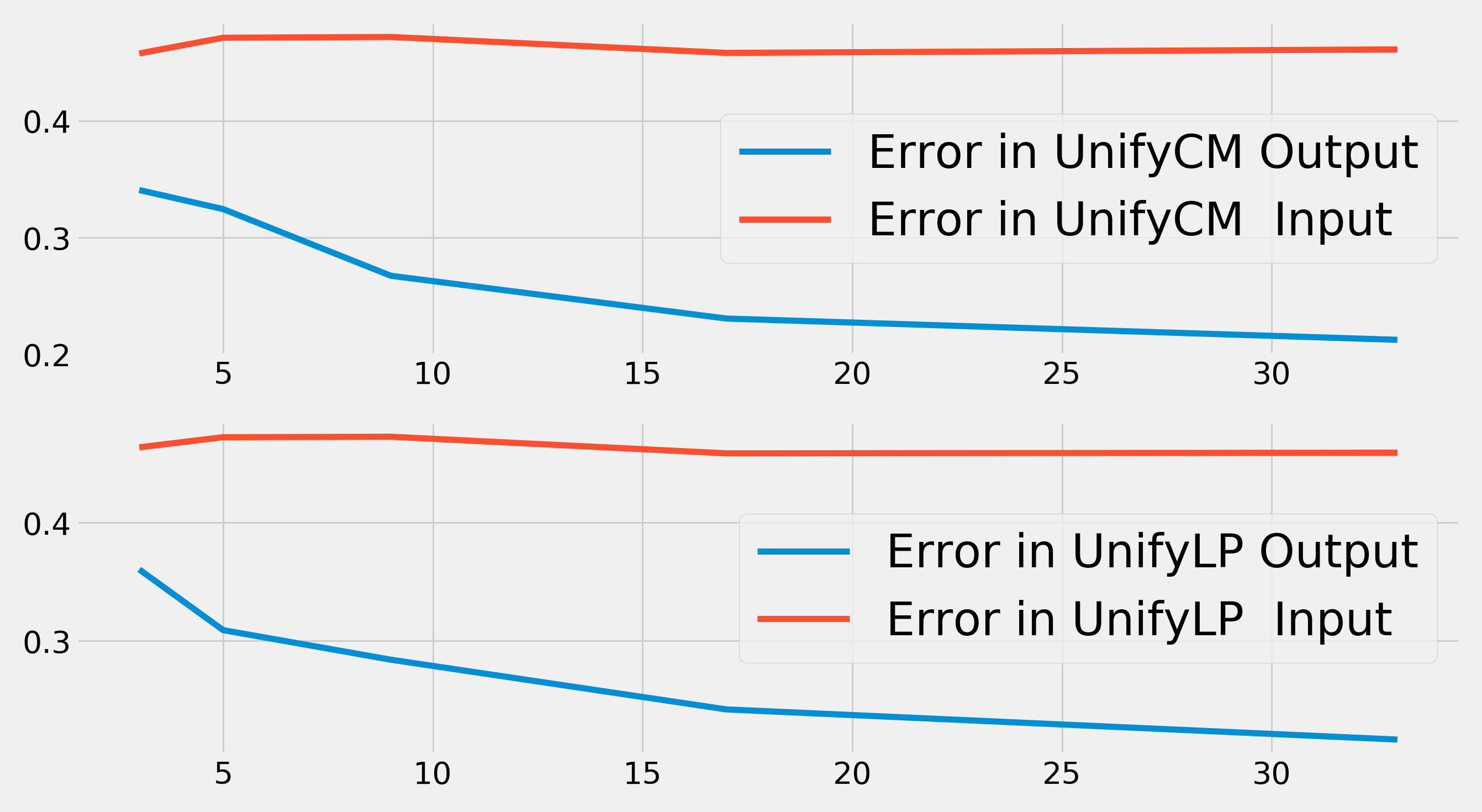}{}
	\includegraphics[width=0.3\columnwidth]{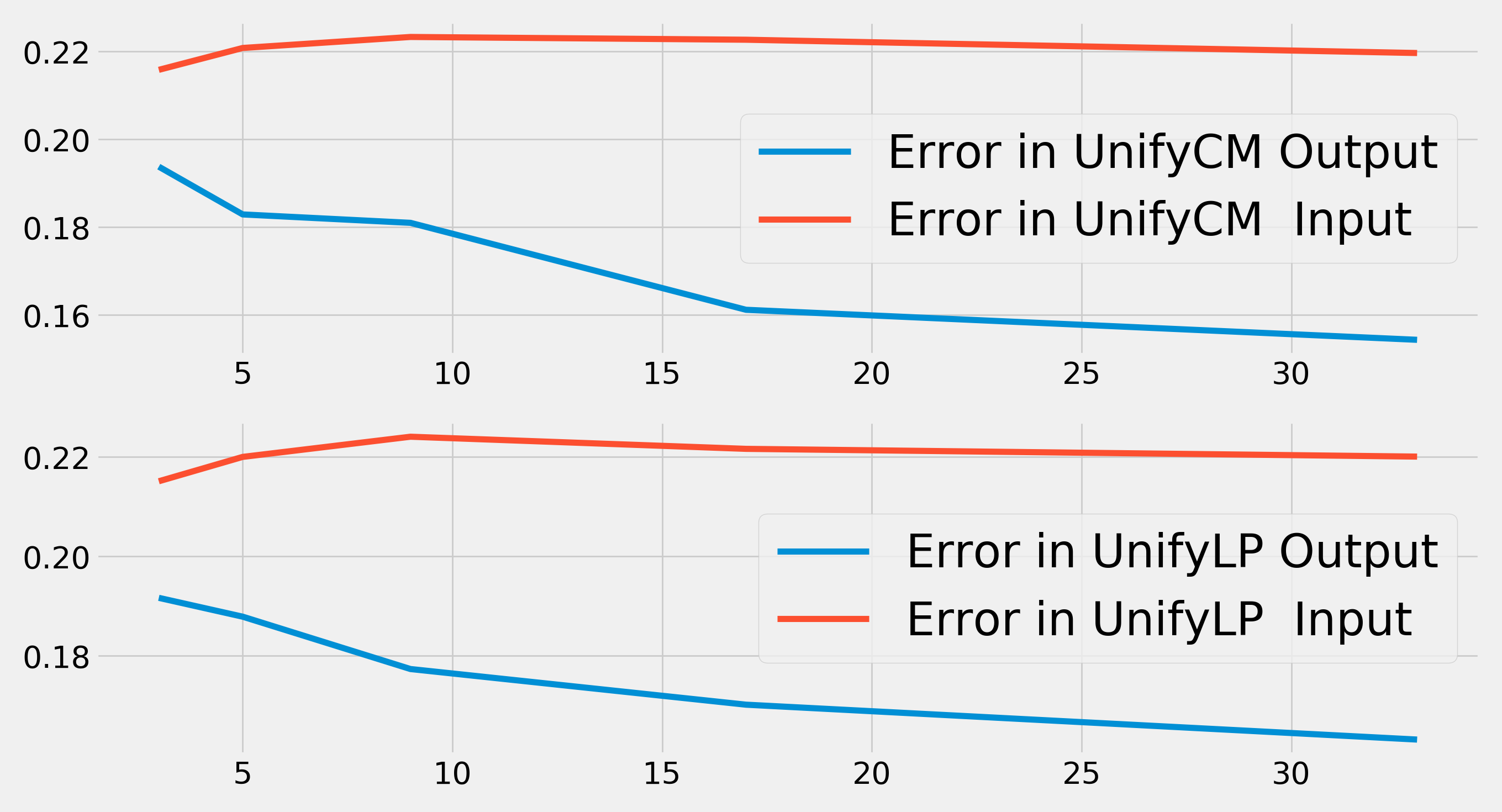}{}
	 \includegraphics[width=0.3\columnwidth]{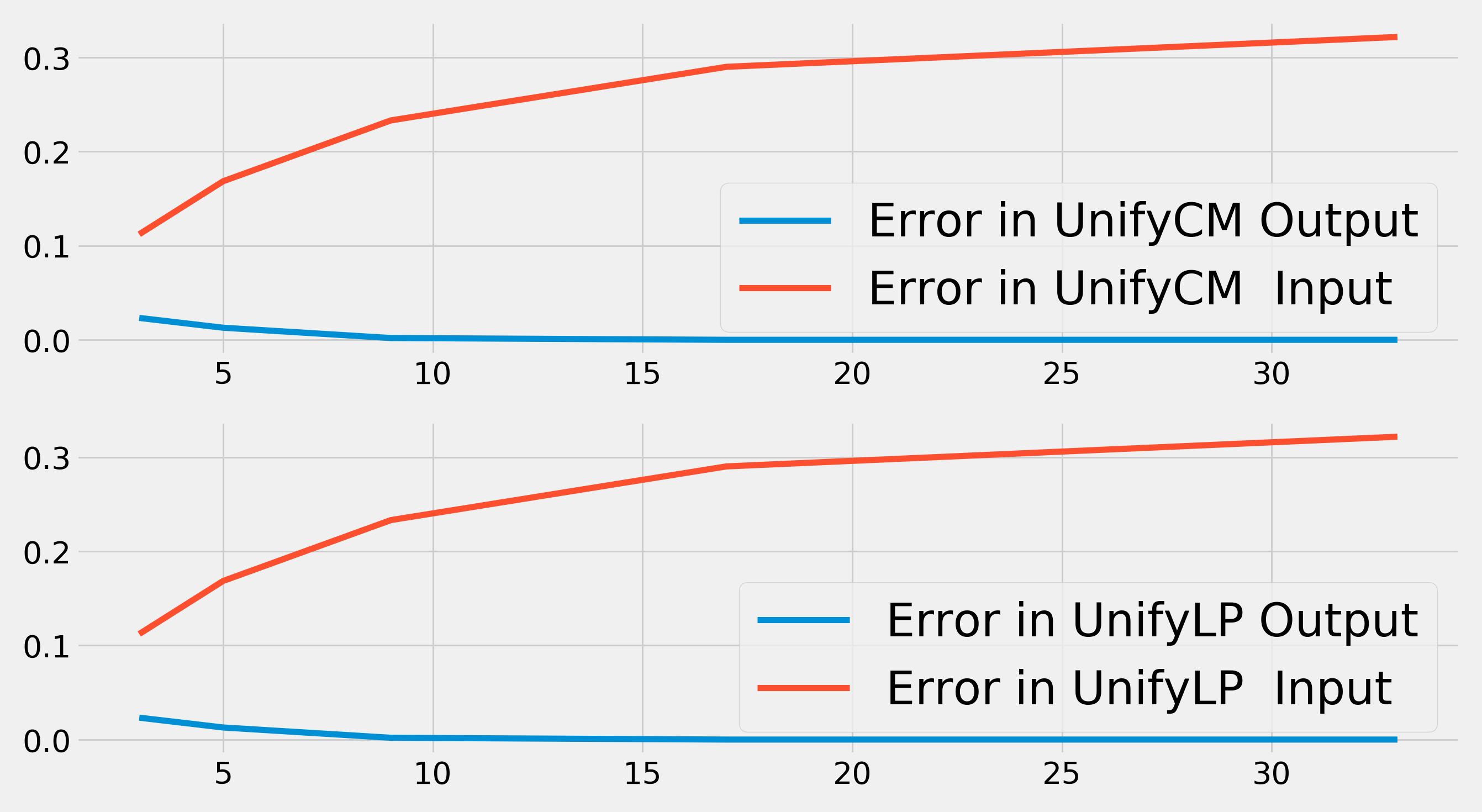}
\caption{Estimation error in community memberships is decreased by UnifyCM and UnifyLP. Setups for each plot are: \emph{left:} $n=100, k=4$ with type-I dynamic, \emph{center:} $n=100, k=4$ with type-II dynamic ($\xi = 0.5$), \emph{right:} $n=500, k=2$ with type-I dynamic.
}\label{fig:ip_vs_op}
\end{figure}


In the next setup (Experiment B), we explore the relative performance of UnifyCM and UnifyLP when compared to Spectral-Mean (\citet{han_consistent_2015,Bhattacharyya2017}). While ~\citet{han_consistent_2015} demonstrate the superiority of Spectral-Mean over a \emph{majority-vote} algorithm (UnifyLP and UnifyCM belong to this family), they only show it on illustrative instances whose parameters are by design below the detectability threshold at each time index. We believe there is merit in investigating how UnifyCM and UnifyLP compare with Spectral-Mean more generally (and not just at the regime where detectability threshold issues arise), especially when the edge probability matrices are time-varying. It was already shown that if these matrices give graphs that wash out the community structure when averaged (e.g., when $W^t = 1-W^{t-1}$ and alternating as a function of $t$) then Spectral-Mean performs poorly. We reproduce this performance in the left plots of Figure~\ref{fig:spectral-mean1} and~\ref{fig:spectral-mean2}. The metric is again based on the normalized mutual information between the estimated memberships and the true memberships. We also show that Spectral-Mean is in general competitive with UnifyCM and UnifyLP when the graphs are generated using a fixed edge probability parameter or when it is stochastic (see the center and right plots in Figures~\ref{fig:spectral-mean1} and~\ref{fig:spectral-mean2}). Since this information is not known a priori, UnifyCM and UnifyLP can be considered being robust compared to Spectral-Mean, as they are similar in performance to the latter in much of the parameter space and do extremely well when averaging out adjacency matrices loses community structure information.

\begin{figure}	
\centering
 \includegraphics[width=0.3\columnwidth]{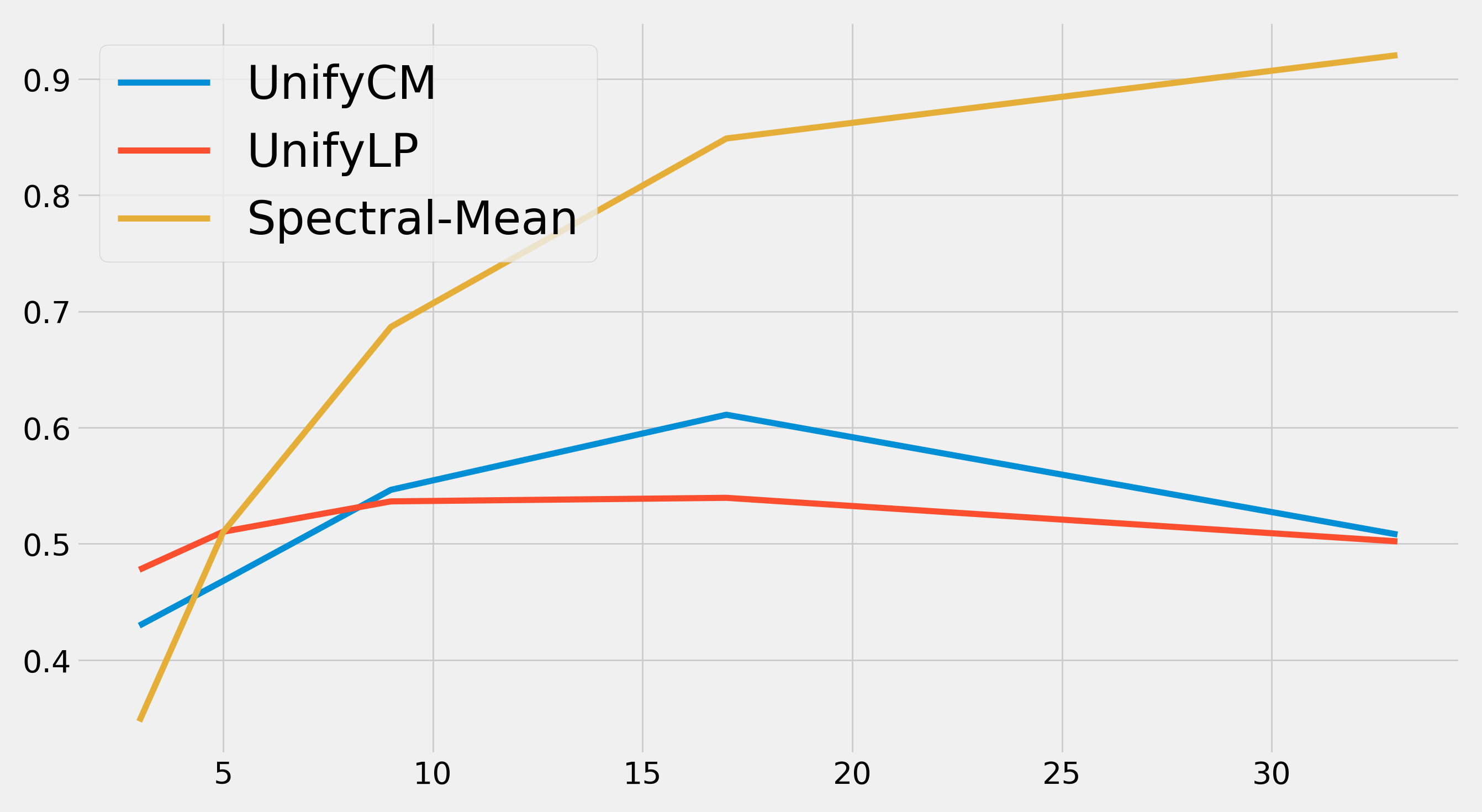}{}
	\includegraphics[width=0.3\columnwidth]{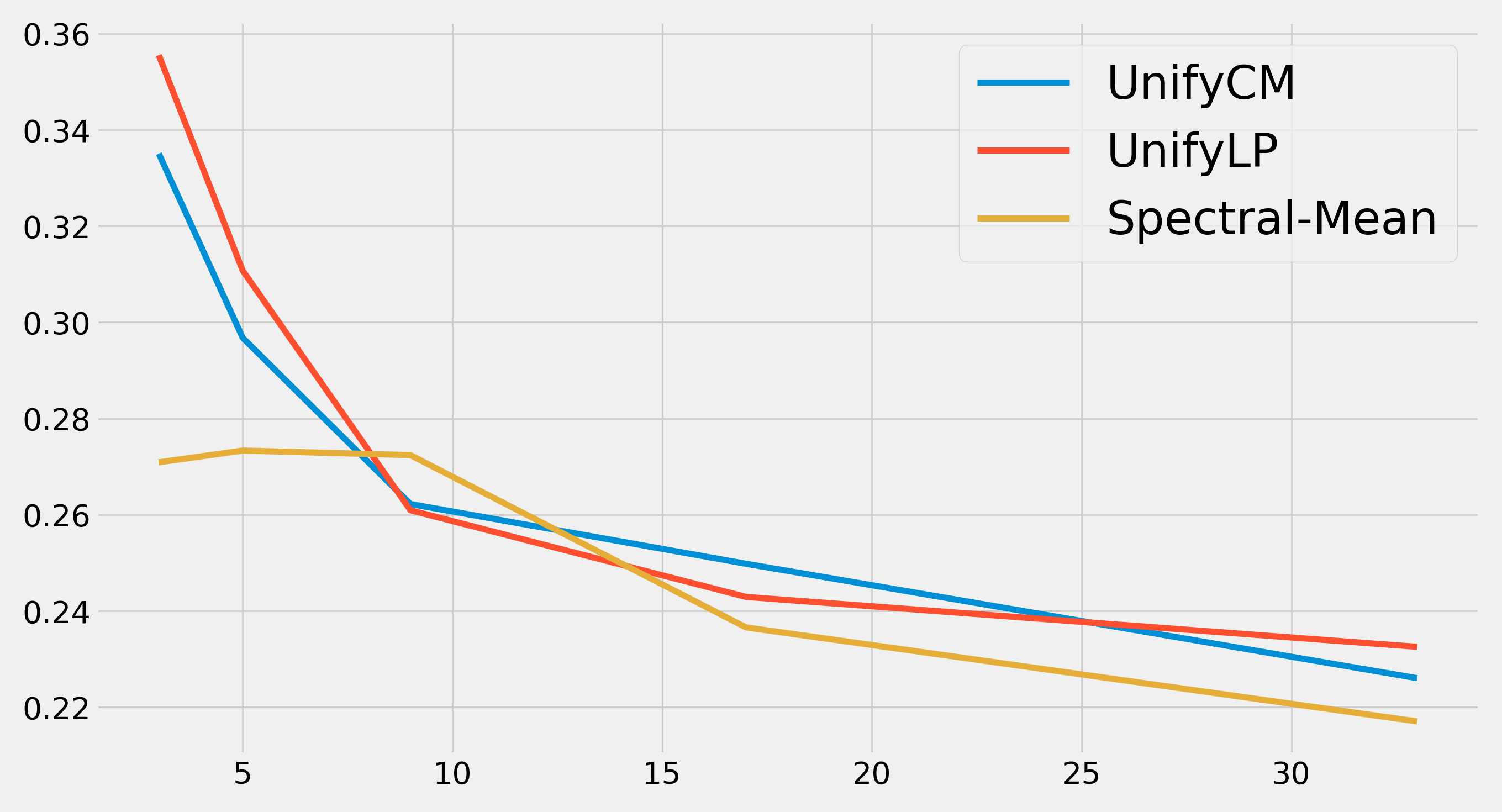}{}
	\includegraphics[width=0.3\columnwidth]{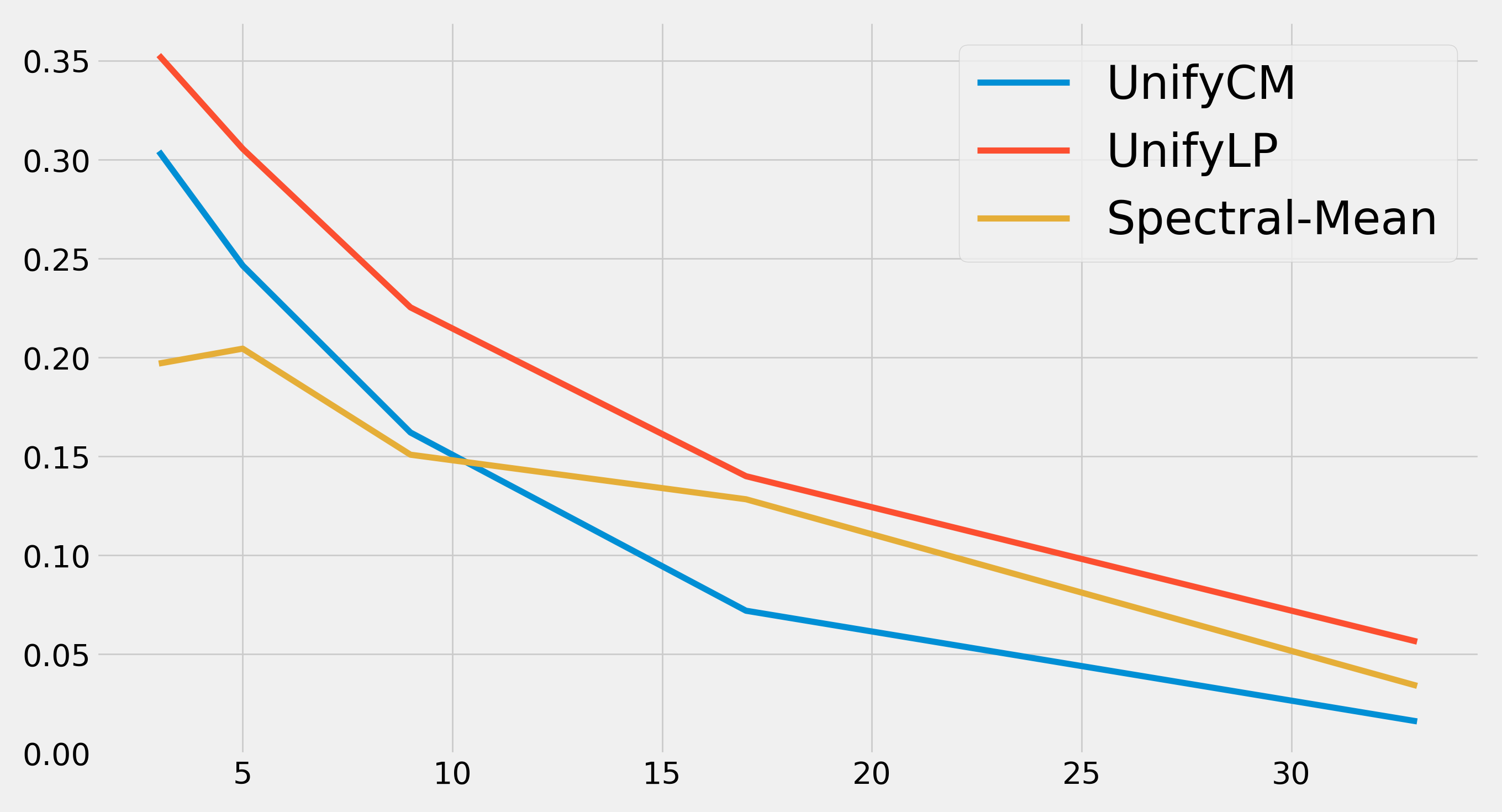}
\caption{Performance (1-NMI) of UnifyCM and UnifyLP compared to Spectral-Mean. The number of nodes is fixed to be $n=100$ and the number of communities is chosen to be $k=4$. {Left}: The edge probability matrices are such that averaging the adjacency matrices washes out the community structure. \emph{Center}: The edge probability matrix is fixed across time. \emph{Right}: The edge probability matrices are drawn i.i.d.}\label{fig:spectral-mean1}
\end{figure}

\begin{figure}	
\centering
 \includegraphics[width=0.3\columnwidth]{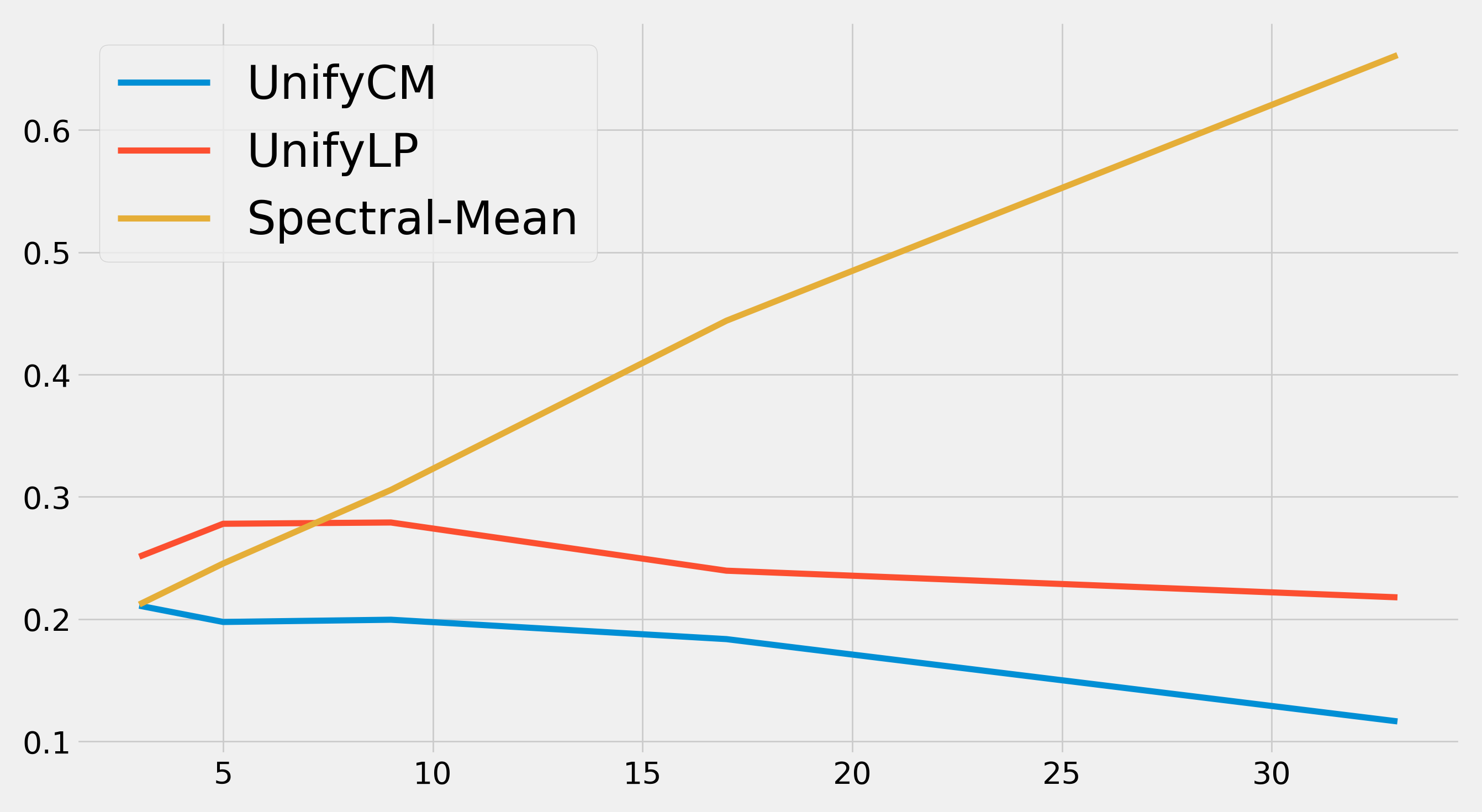}{}
	\includegraphics[width=0.3\columnwidth]{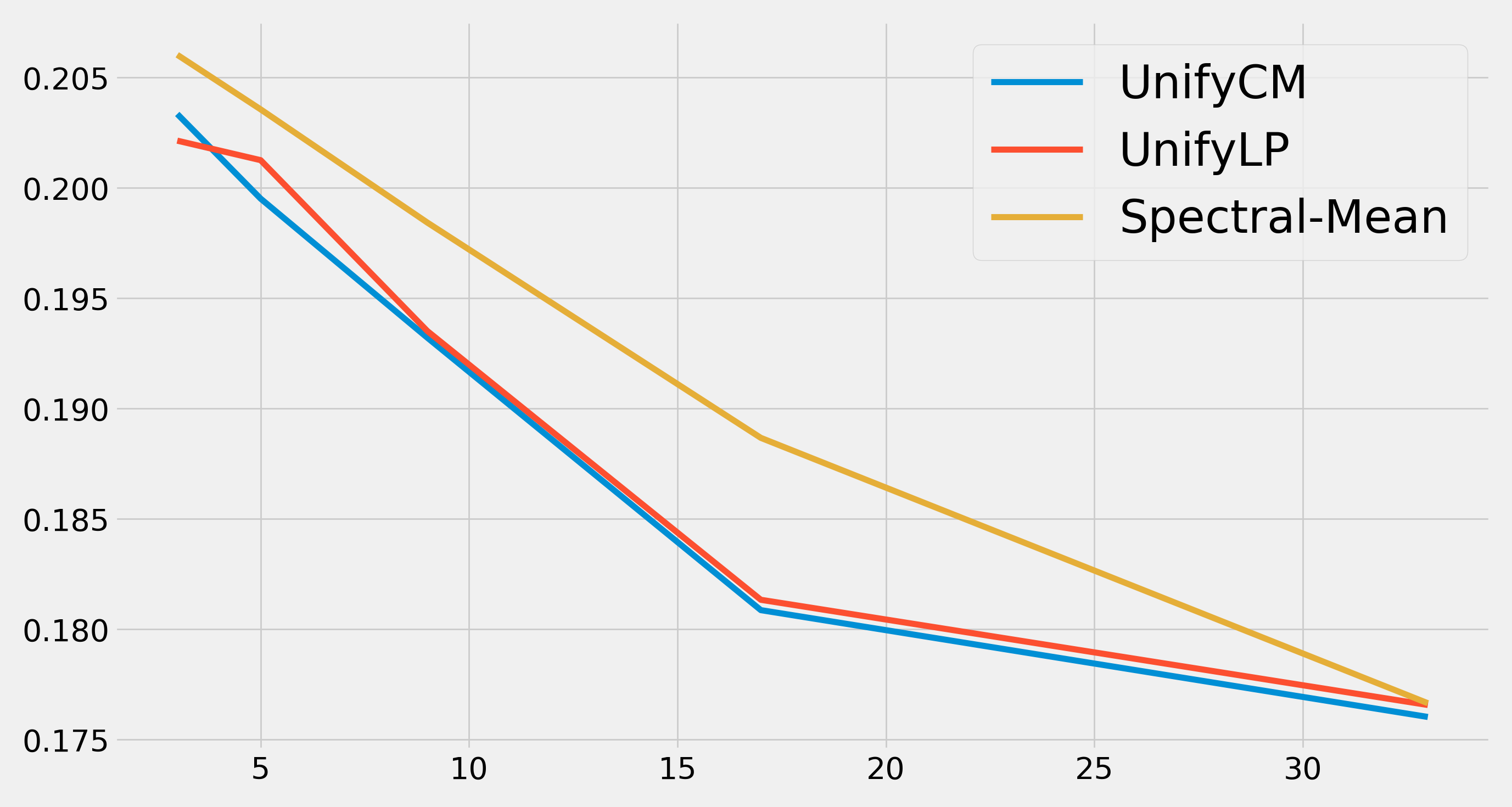}{}
	\includegraphics[width=0.3\columnwidth]{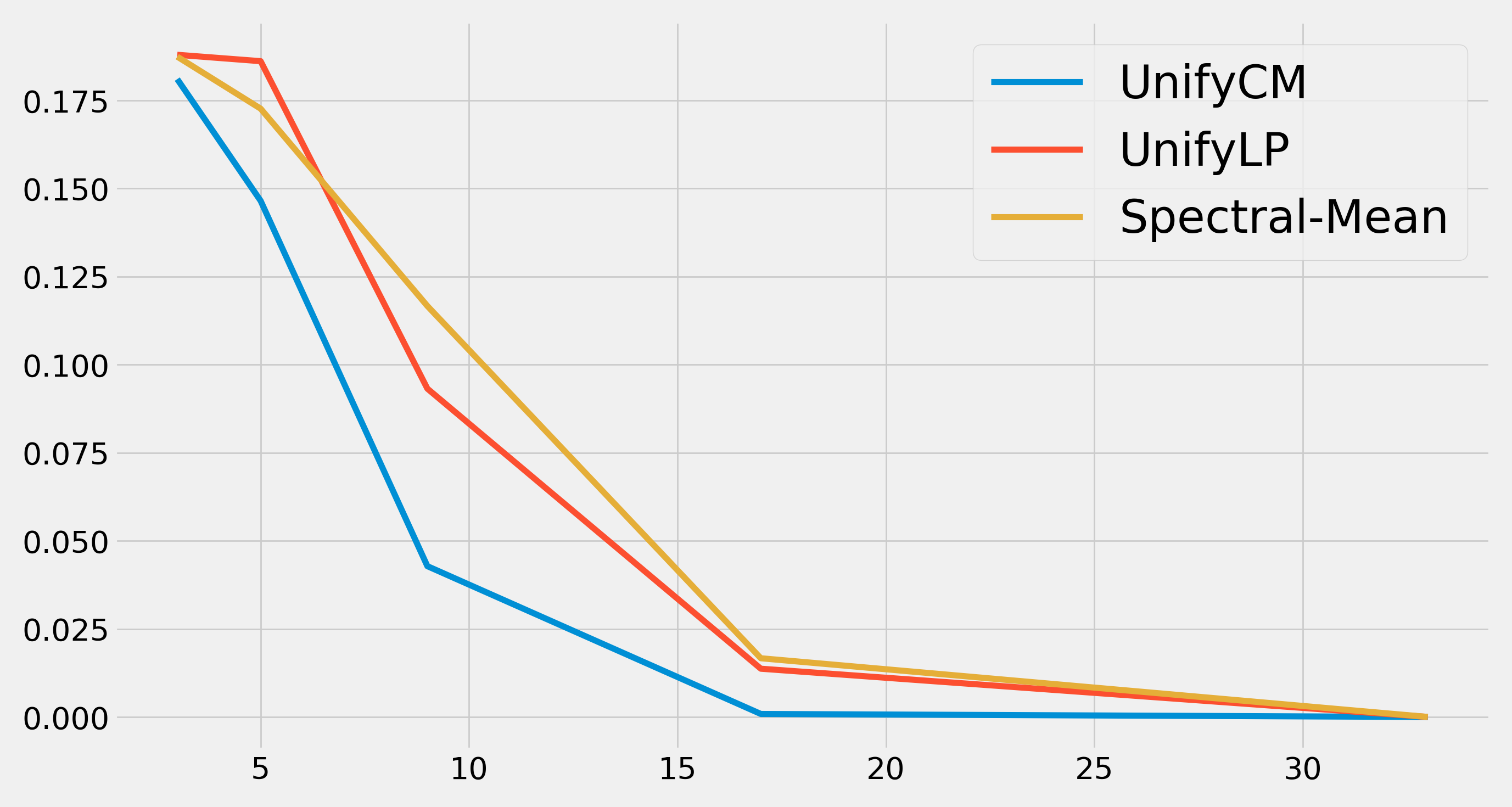}
\caption{Performance (1-NMI) of UnifyCM and UnifyLP compared to Spectral-Mean. The plots are similar to the ones in Figure~\ref{fig:spectral-mean1}. In these plots, the number of nodes is fixed to be $n=500$ and the number of communities is chosen to be $k=4$. {Left}: The edge probability matrices are such that averaging the adjacency matrices washes out the community structure. \emph{Center}: The edge probability matrix is fixed across time. \emph{Right}: The edge probability matrices are drawn i.i.d.}\label{fig:spectral-mean2}
\end{figure}


In the third experiment (Experiment C), we assess the quality of the estimates of community memberships using Algorithm~\ref{Inference 2nd Model} under the model in Section~\ref{sec:group-changing-model}. Three sets (10 Monte Carlo runs each) of synthetic graph sequence instances were generated with the following configurations: $(n,k,\xi) = (500,2,0.2),(500,2,0.8)$ and $(1000,4,0.2)$. The spatial edge probability matrix $W$ was chose to be $.3\mathbb{I} + .2\textbf{1}\textbf{1}'$ for simplicity. As discussed before, we assume exogenous access to the majority/minority labels of nodes at time $t-1$ while predicting the labels of all nodes at time $t$. Since the majority nodes don't change labels, their prediction is relatively simple. For the minority nodes however, although they may have changed to any other community uniformly at random, executing Algorithm~\ref{Inference 2nd Model} allows us to recover their communities. In fact, it does this using the same SBM routines (spectral clustering) that were used in UnifyCM and unifyLP. When we next move to time indices, $t$ and $t+1$, the nodes which were minorities at time $t-1$ are removed from the graph sequence. Thus, the graph sequence size decreases as estimation time index increases. In the left top subplot in Figure~\ref{fig:changing}, we are showing the size of the majority nodes and the minority nodes at time t-1 as a function of time index $t$. The community memberships of the latter set of nodes is estimated for time $t$ and the error in this estimate is plotted in the left bottom subplot. The estimation error in $\xi$  dies down quickly as we get more snapshots to estimate it versus $t$, and we omit plotting this. The center and right subplots of Figure~\ref{fig:changing} show similar performances for the remaining two configurations. All estimates are averaged over 10 Monte Carlo runs, where the randomness is in the sequence of adjacency matrices generated in each run.

\begin{figure}	
\centering
 \includegraphics[width=0.3\columnwidth]{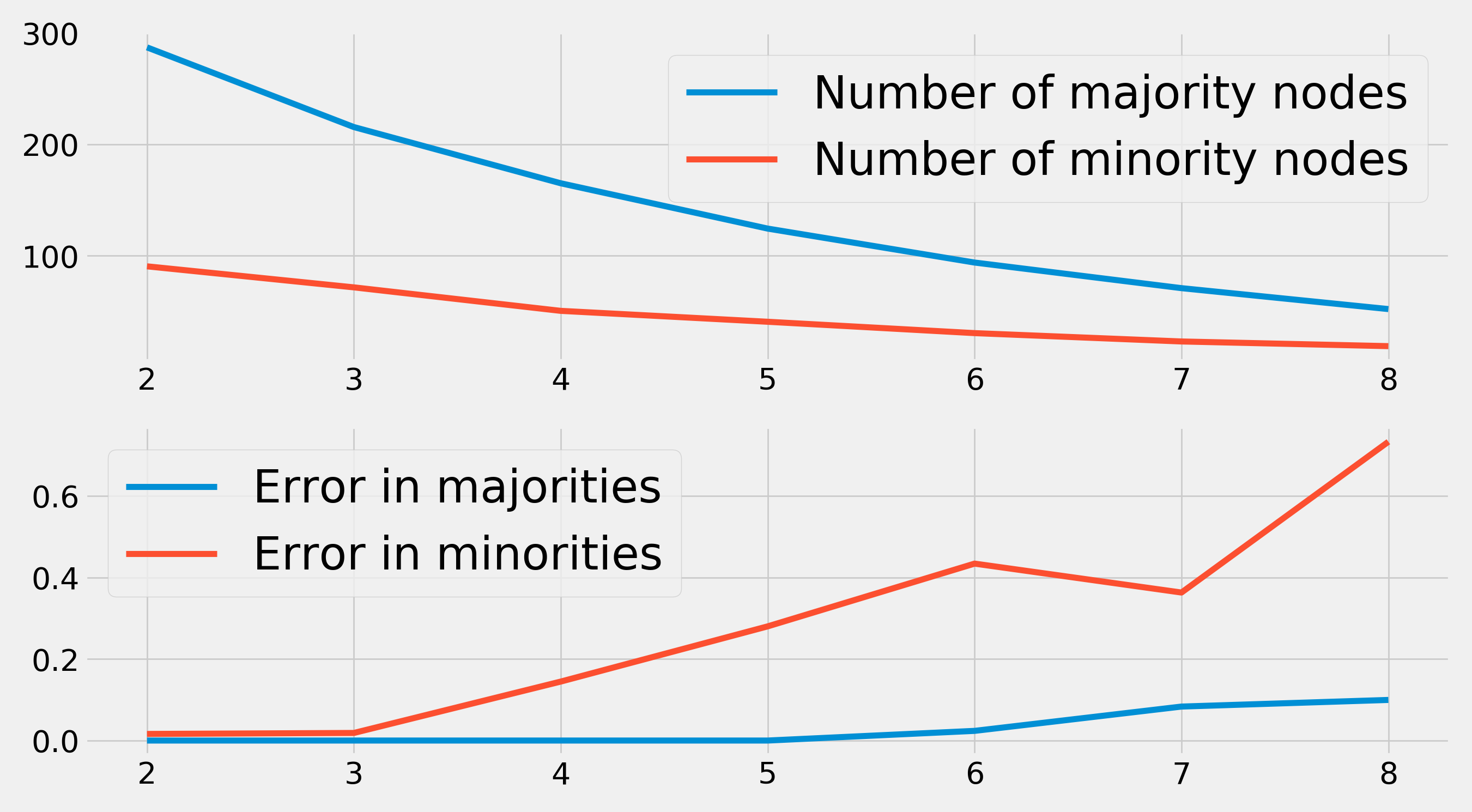}{}
	\includegraphics[width=0.3\columnwidth]{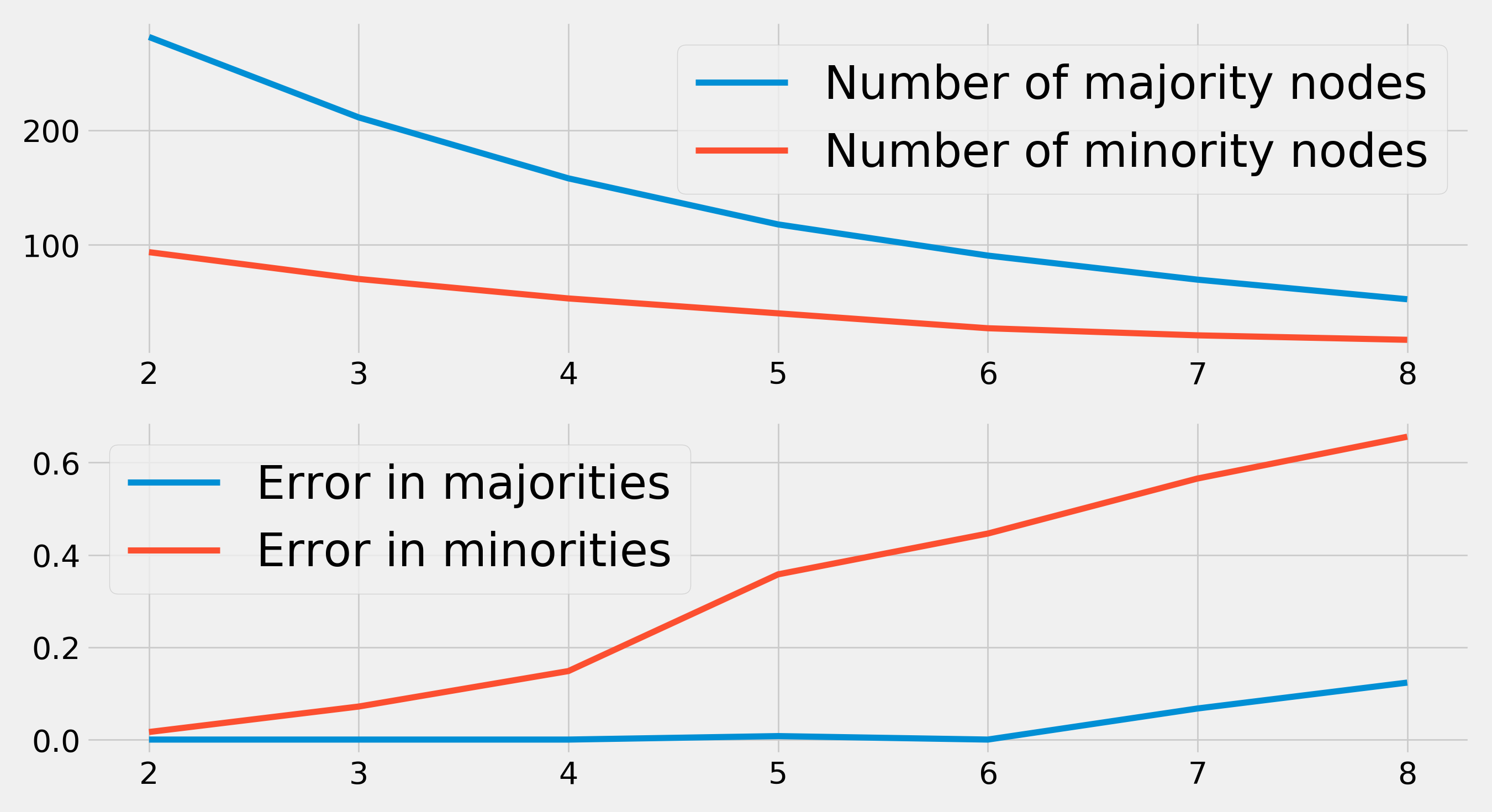}{}
		\includegraphics[width=0.3\columnwidth]{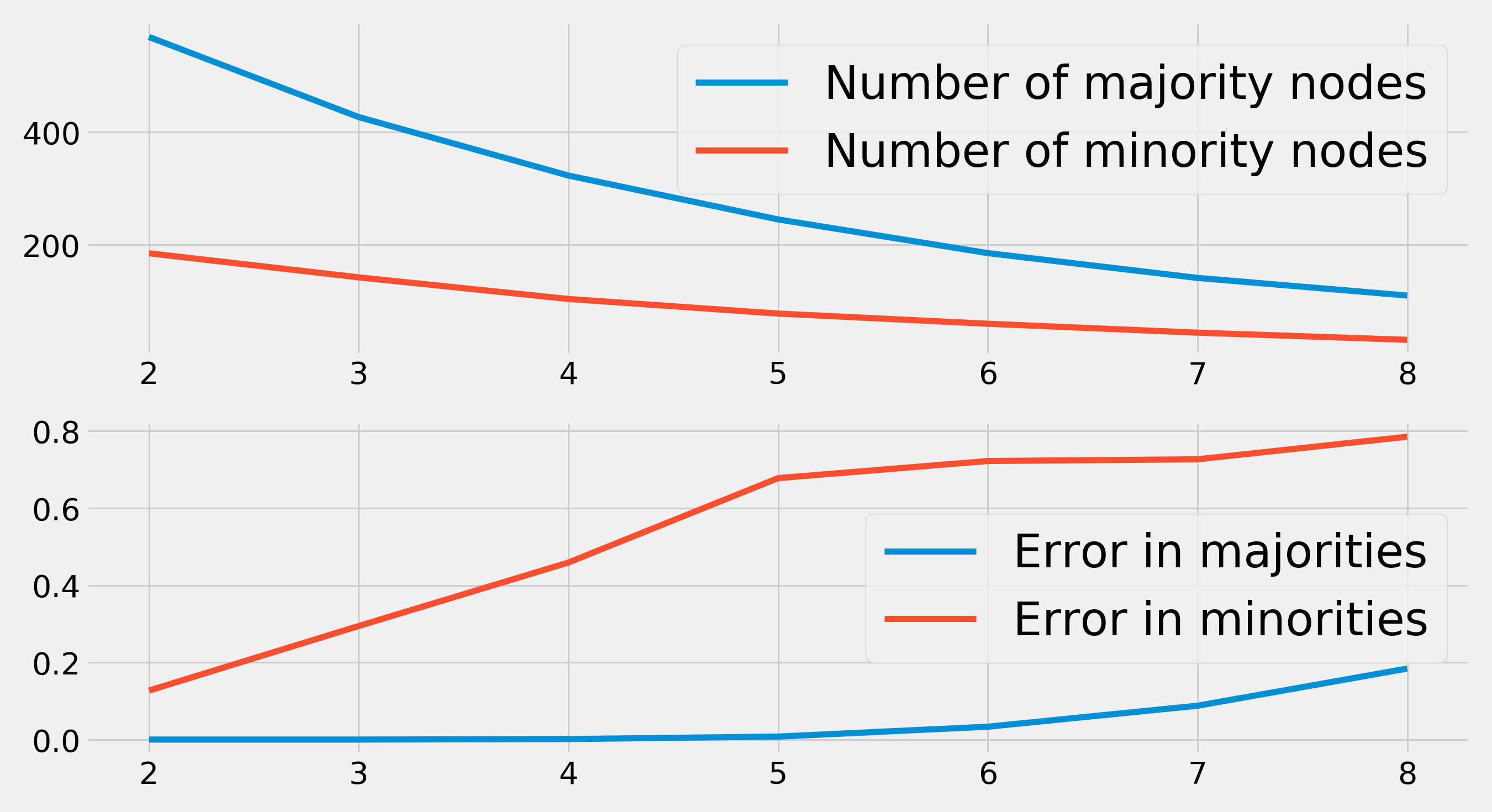}
\caption{Number of nodes (top) and estimation error (1-NMI, bottom) as a function of time index $t$ in the community and link persistence model of Section~\ref{sec:group-changing-model}. Here, the errors are increasing because the number of nodes used for estimating minority node memberships is decreasing. The decrease is due to the fact that nodes that once change communities cannot be reused at a future time in our approach because they don't preserve the single graph SBM structure that is needed for inference. The configurations are: left: $(n,k,\xi) = (500,2,0.2)$, center: $(n,k,\xi) = (500,2,0.8)$, and right: $(n,k,\xi) = (1000,4,0.2)$.}\label{fig:changing}
\end{figure}

In the next two experiments (Experiments D and E), we explore the performance of Algorithms~\ref{Inference general Fixed model} and~\ref{Inference Fixed Barrucca} on synthetic instances across a variety of parameter regimes (e.g., graph size, number of communities etc.). We first start with the setting when the edge probability matrix is fixed (Experiment D). In Figure~\ref{fig:fixed_group_fixed_w1}, we plot the temporal dependence of errors in our estimates of the community memberships and the model parameters for three configurations, viz., (a) $(n,k) = (500,2)$ with type-I dynamic ($\wrs=\bigl(\begin{smallmatrix}
  0.3 & 0.2\\
  0.2 & 0.3
 \end{smallmatrix}\bigl)$ and  $\mu=
\bigl(\begin{smallmatrix}
  0..6 & 0.4\\
  0.4 & 0.6
 \end{smallmatrix}\bigl)$), (b) $(n,k) = (1000,2)$ with type-I dynamic (parameters same as before), and (c) $(n,k) = (2000,2)$ with type-I dynamic (parameters same as before). In all three cases, the errors die down quickly, as seen in the figure. In particular, as the number of nodes increases, the community estimation problem becomes easier, keeping everything else the same. Next, to explore if the parameters themselves influence estimation errors, we vary $\xi$ while keeping $(n,k) = (500,2)$. In particular, we set $\xi = 0.2, 0.5$ and $0.8$ respectively (and $\wrs=\bigl(\begin{smallmatrix}
  0.3 & 0.2\\
  0.2 & 0.3
 \end{smallmatrix}\bigl)$ is kept fixed). As seen in Figure~\ref{fig:fixed_group_fixed_w2}, the estimation error profiles are similar and fairly insensitive to the values of the parameter changed, across these settings. Experiments with changing $\mu,\wrs$ had similar trends and have been omitted here.
 
 \begin{figure}	
 \centering
 \includegraphics[width=0.3\columnwidth]{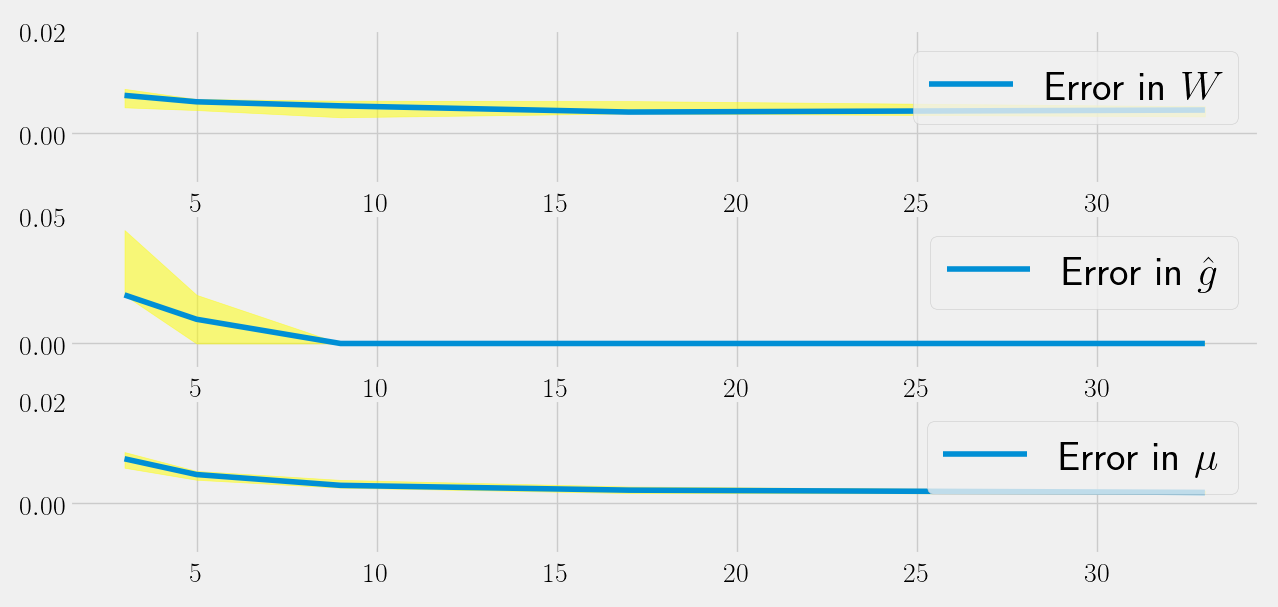}{}
	\includegraphics[width=0.3\columnwidth]{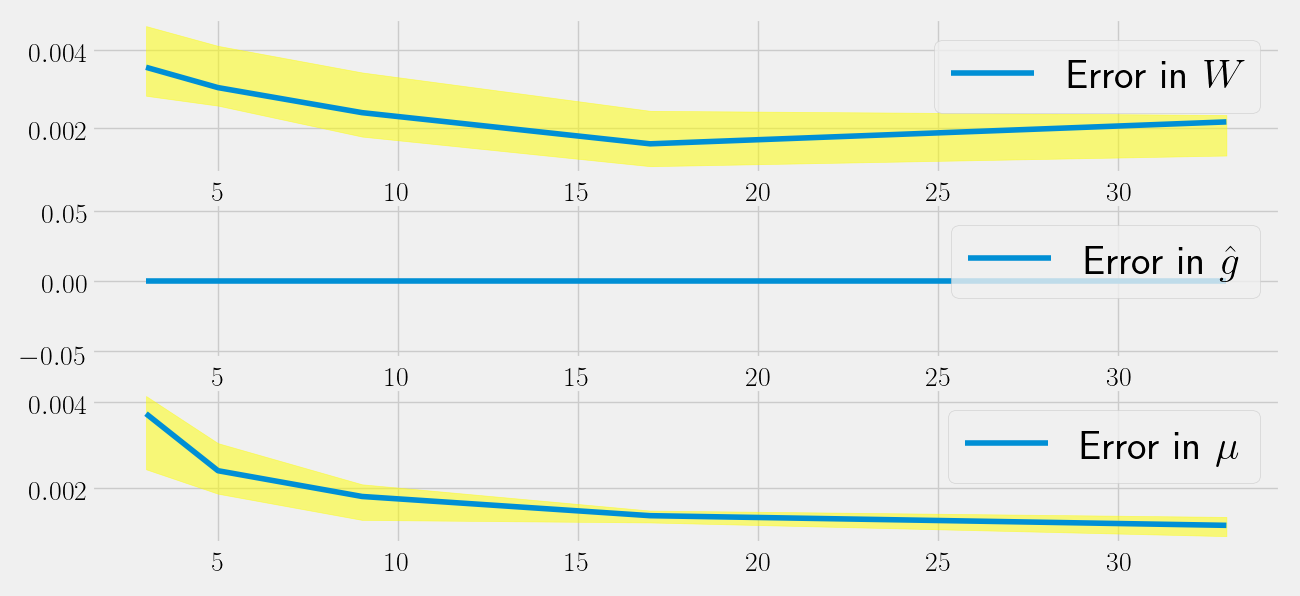}{}
		\includegraphics[width=0.3\columnwidth]{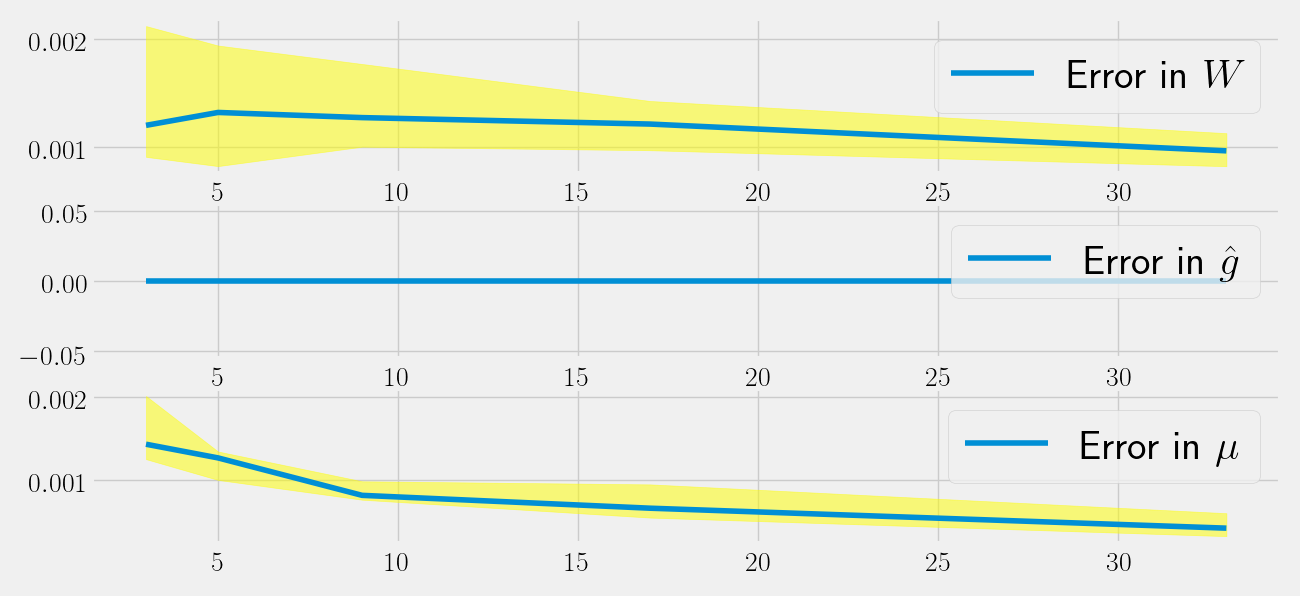}
\caption{Estimation error as a function of time index in the link persistence model. UnifyCM was used to estimate community memberships. Left: $(n,k) = (500,2)$ with type-I dynamic. Center: $(n,k) = (1000,2)$ with type-I dynamic. Right: $(n,k) = (2000,2)$ with type-I dynamic. }\label{fig:fixed_group_fixed_w1}
\end{figure}

 \begin{figure}	
 \centering
 \includegraphics[width=0.3\columnwidth]{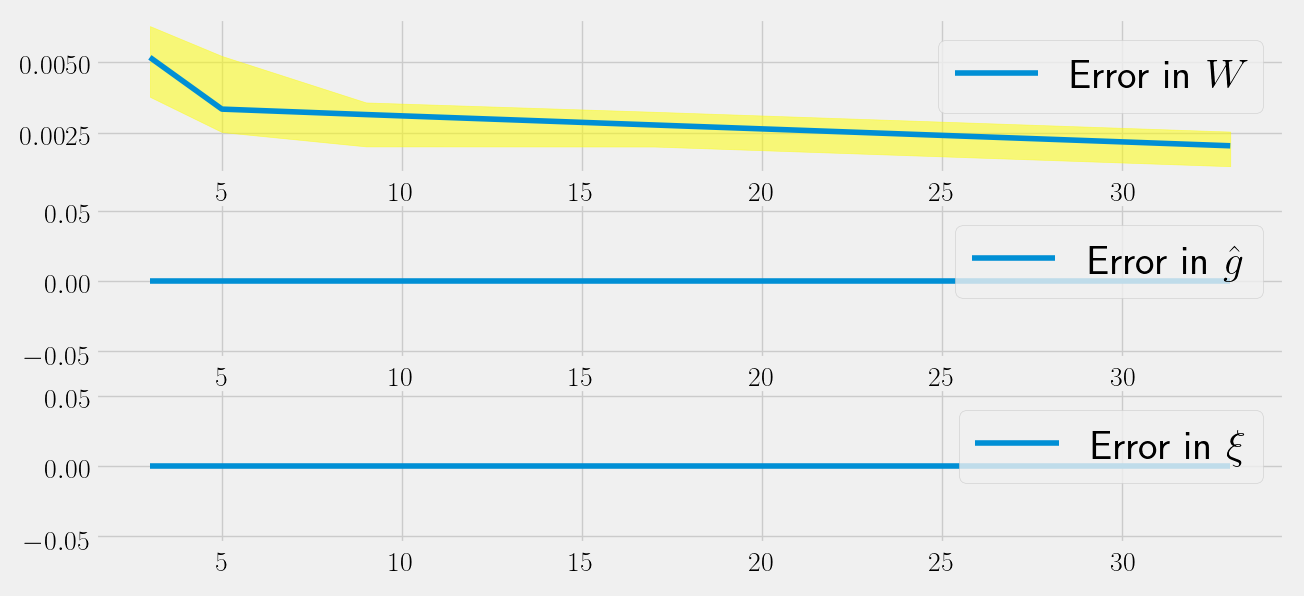}{}
	\includegraphics[width=0.3\columnwidth]{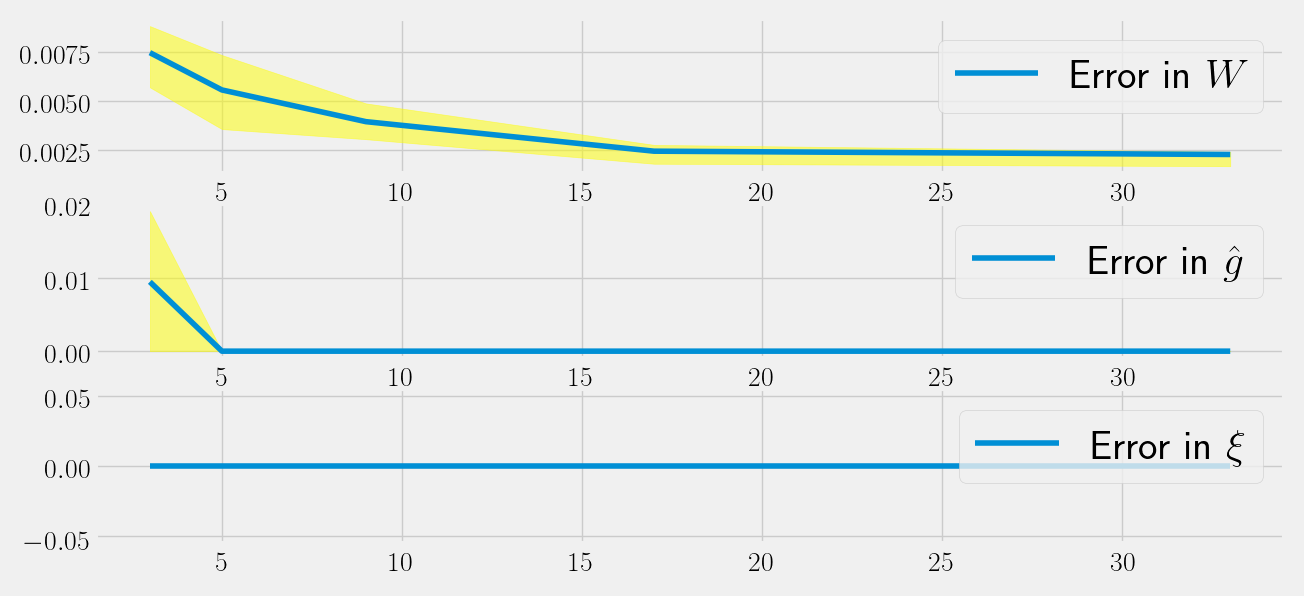}{}
		\includegraphics[width=0.3\columnwidth]{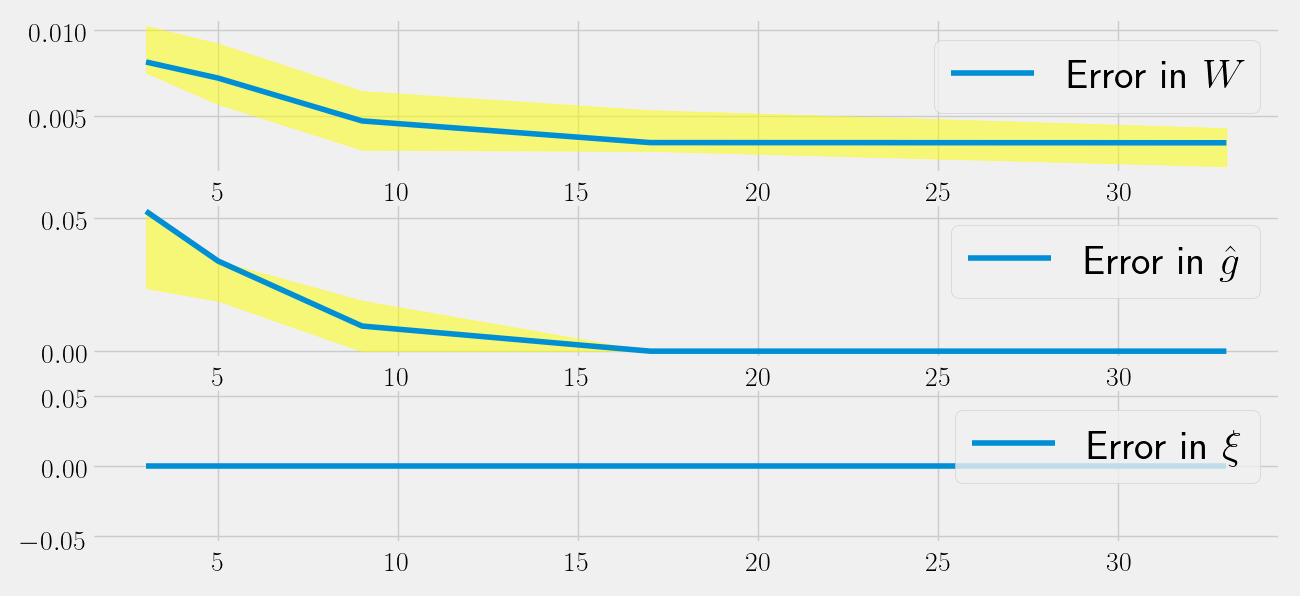}
\caption{Estimation error as a function of time index in the link persistence model for type-II dynamic as we change the parameter $\xi$ with $(n,k) = (500,2)$ kept fixed. Left: $\xi = 0.2$. Center: $\xi = 0.5$. Right: $\xi = 0.8$. Again, UnifyCM was used to estimate community memberships.}\label{fig:fixed_group_fixed_w2}
\end{figure}

We relax the fixed edge probability matrix restriction (Experiment E) and study the performance of our methods on parameter and community estimation. The edge probability matrices themselves were generated i.i.d (i.i.d. draws allow for SBM interpretability at each time-step, and SBM routines can be used). The plots in Figure~\ref{fig:fixed_group_changing_w} show two configurations of the type-I dynamic. In the first configuration, $(n,k) = (500,2)$ and in the second configuration $(n,k) = (500,4)$. In both plots, the errors for estimated $W$ matrices are not expected to decrease because we have one additional matrix to estimate at each new time step. The error on $\mu$ on the other hand decreases (albeit slowly) as we get more snapshots. One reason for this could be that as time steps increase, the MLE optimization objective is less sensitive to the choice of $\mu$ (i.e., becomes flat). In both this experiment and the previous one, the estimation errors in the community memberships become zero very quickly.

 \begin{figure}	
 \centering
 \includegraphics[width=0.45\columnwidth]{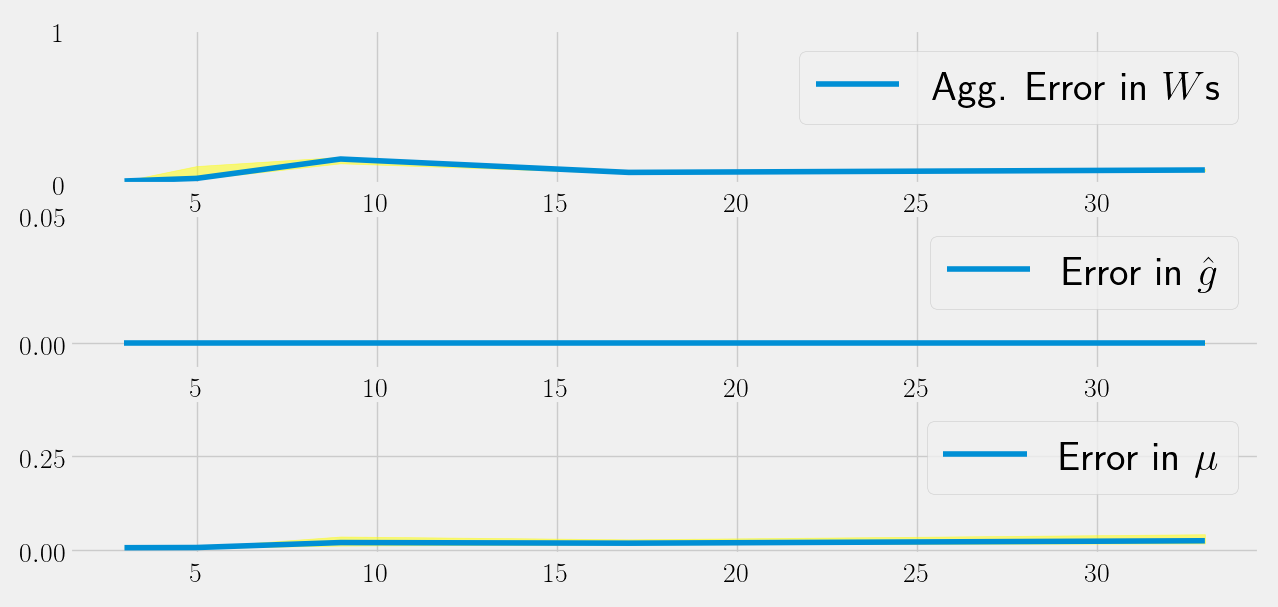}{}
	\includegraphics[width=0.45\columnwidth]{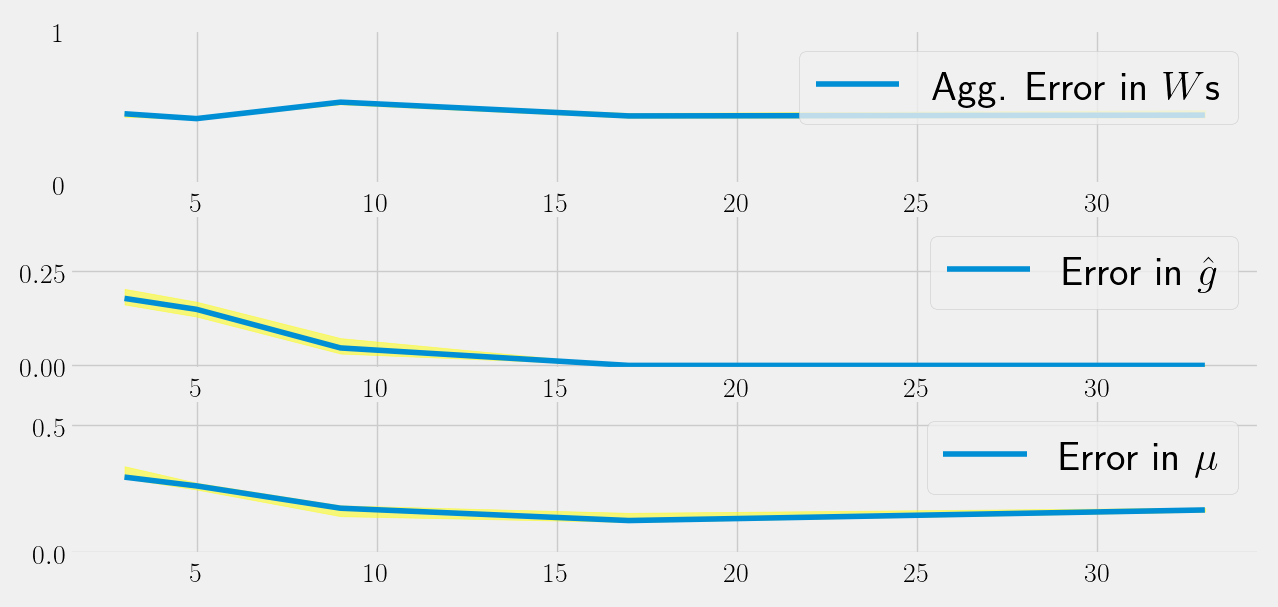}
\caption{Estimation error as a function of time index in the general link persistence model ($\wrs^t$ drawn independently at each time step). UnifyCM was used to estimate community memberships. Left: $(n,k) = (500,2)$. Right: $(n,k) = (500,4)$. Both are with type-I dynamic. Estimation errors on the aggregate of $\wrs$ estimates don't decrease as we estimate an additional $O(k^2)$ parameters with each increasing time-step. }\label{fig:fixed_group_changing_w}
\end{figure}

In the next setup (Experiment F), we compare the community recovery quality of our algorithms when compared to the DSBM approach (also called HM-SBM). DSBM (and its variant SBTM) assume a latent stochastic process governing the evolution of the edge probability matrices. The authors propose inference on this class of models using Kalman filtering (KF) and greedy local search. We first generate graph sequences based on a DSBM model. These sequences are input both our methods and to the KF based inference procedure. Figure~\ref{fig:comparison_dsbm} plots the error in the community memberships estimated at each time step. As can be observed, by the end of the sequence, our methods are able to perform approx. 50\% percent better. Spectral-Separate is a baseline that estimates communities based on the given snapshot's adjacency matrix only, and it performs considerably worse than the alternatives.

\begin{figure}
\centering
 \includegraphics[width=0.5\columnwidth]{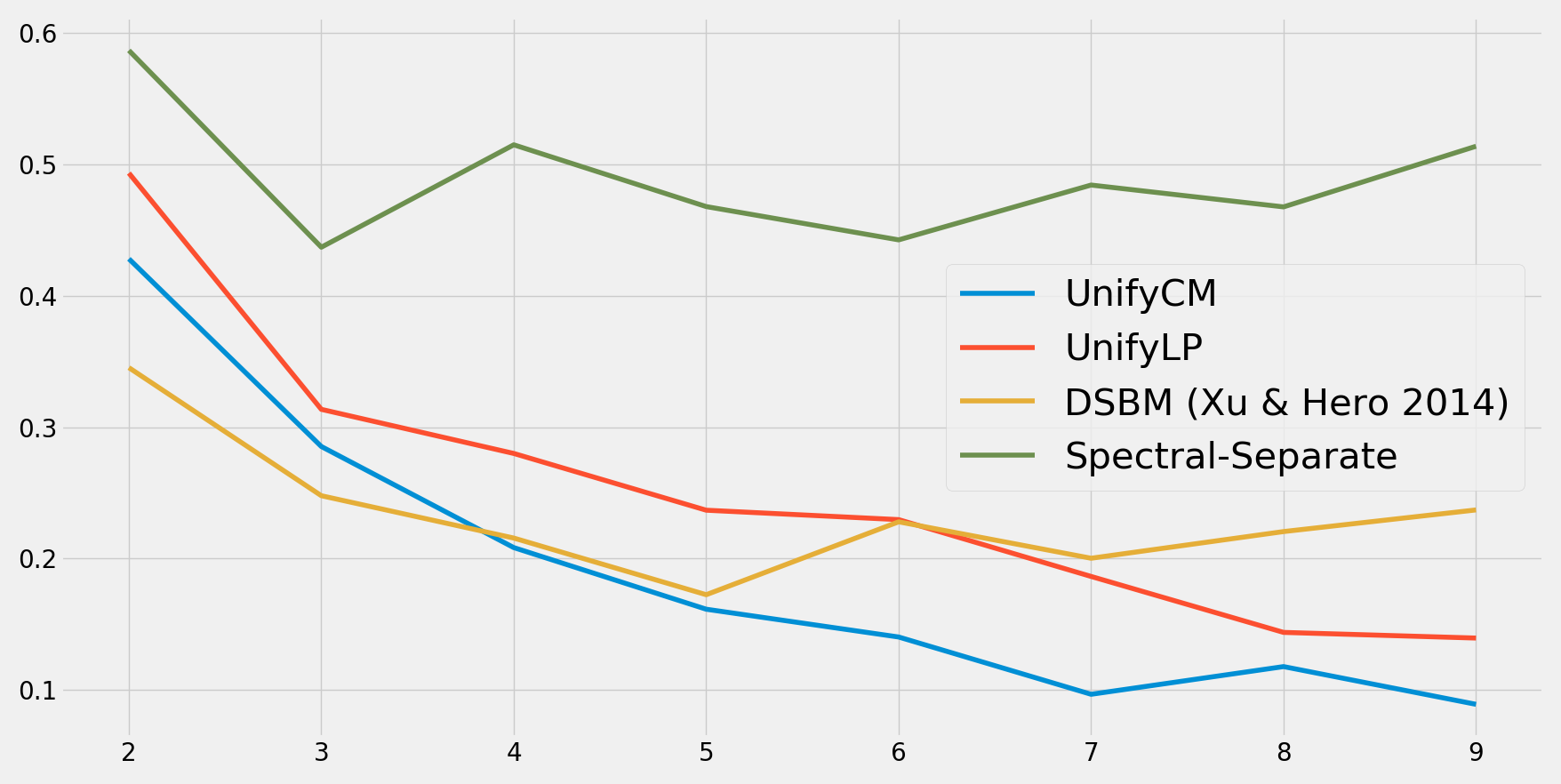}{}
\caption{Estimation error (1-NMI, lower is better) in community memberships as a function of time index $t$ when the data is generated according to a DSBM model.}\label{fig:comparison_dsbm}
\end{figure}

Finally, we assess the link prediction capability of our model using the Enron email dataset~\cite{park2008scan} and the Facebook friendship dataset~\cite{viswanath-2009-activity} (Experiment G). For the former dataset, we use 7 day periods to generate the adjacency matrix of communication between employees (184) with various roles. We also truncate a few weeks at the beginning and at the end where the communication frequency is low (this ultimately leads to a $120$ time step graph sequence). The latter dataset is of a regional network in the USA that was crawled between 2006-2009. It consists of information about when two users became friends. We follow the same preprocessing procedure as listed in~\cite{junuthula2018leveraging} and obtain a sequence of 9 graphs with 1988 nodes. For both datasets, we set up the link prediction auxiliary task as follows. The sequence of adjacency matrices obtained from the datasets is used to estimate the parameters of our type-I dynamic model with a fixed edge probability matrix. In particular, at time $t$, we use the previous $l$ graph snapshots to fit our model ($l=4$ for Enron and $l=2$ for Facebook). The estimated class memberships and the edge probability matrix  are then used to predict the edge formation probability between every pair of nodes at time $t$. We repeat these steps till we reach the end of the sequence by moving forward one time step. Since the actual number of communications is sparse, we use AUC as a measure to access our link prediction accuracy. As Figure~\ref{fig:real_data_roc} shows, we obtain an AUC value of $0.75$ without tuning any parameters for the Enron dataset. There is less predictability in the Facebook dataset, giving us an AUC score of $0.57$. These numbers are somewhat lower than those reported using more expressive models, such as the DSBM, which are very inefficient to infer compared to our methods. Nonetheless, the ROC scores indicate that our models are capturing relevant information from the graph sequences and are able to perform respectably without any finetuning. Finally, note that the reported numbers are for a single sample path for both datasets, and care is needed while drawing general conclusions about these datasets.

\begin{figure}	
\centering
 \includegraphics[width=0.45\columnwidth]{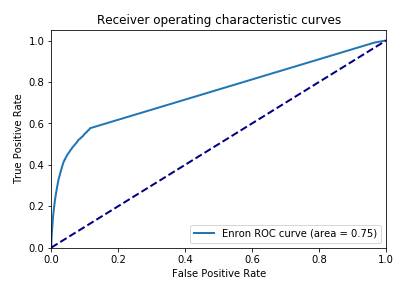}{}
 \includegraphics[width=0.45\columnwidth]{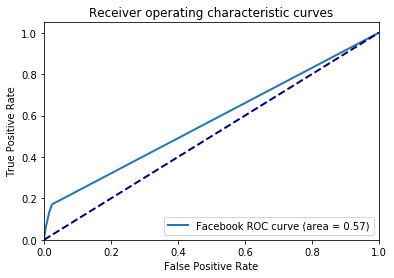}
\caption{The Receiver operating characteristic (ROC) curve for the link prediction auxiliary task is plotted here. The prediction is performed using a sequence of models fitted according to the type-I dynamic. Left: Enron email data. Right: Facebook friendship network.}\label{fig:real_data_roc}
\end{figure}

%% file: sec_conclusion.tex
\section{Conclusion}\label{sec:conclusion}

In this work, we developed new block structure based statistical models for time-series graph sequences, and proposed new algorithms for inference that are able to exploit efficient methods designed for the single graph SBM model. We captured both link and community membership persistence across time. Inference methods for our models rely on estimating community memberships independently for each graph observation and then combining these outputs, taking into account errors and permutations. Because of this modular aspect of the algorithms, our methods can directly take advantage of the computational and statistical advances made for the single graph SBM model.

%% file: sec_appendix_proofs.tex
\section{Proofs}
In this section, we provide the proofs of various claims.
\subsection{Proof of Lemma~\ref{each_SBM}}
\begin{proof}
Let $\wt^t_{g_ig_j}:=P(\adj^t_{ij}=1|\pzo^t_{\gro_i\gro_j},\poo^t_{\gro_i\gro_j})$ denote the conditional (conditioned on the community memberships) probability of a link between two nodes at time $t$. This probability can be observed as the link formation probability of a new SBM structure because it only depends on the community memberships of the relevant nodes defined by the link. This is a simple consequence of assuming that the stochastic process for each link is independent and that the community memberships are fixed. The new SBM obtained from marginalizing link variables has the same block structure as the graph sequence model in Equation~\ref{MLE_1}. Its parameters are related to the original model parameters $(\wrs^1, ..., \wrs^T, x)$ as:
\begin{align} \label{SBM_prf}
\wt^t_{g_ig_j}&=\sum_{\adj^{t-1}_{ij}} P(\adj^t_{ij}=1, \adj^{t-1}_{ij}|\gro_i,\gro_j)= p^t_{1\rightarrow 1}\wt^{t-1}_{\gro_i\gro_j}+p^t_{0\rightarrow 1}(1-\wt^{t-1}_{\gro_i\gro_j}), \nonumber
\\
&=\poo^t_{\gro_i\gro_j}\wt^{t-1}_{\gro_i\gro_j}+\pzo^t_{\gro_i\gro_j}(1-\wt^{t-1}_{\gro_i\gro_j}), \nonumber
\\
&= \prod_{\tau=2,...,t} \Big(\poo^\tau_{\gro_i\gro_j}-\pzo^\tau_{\gro_i\gro_j}\Big) \wt^1_{\gro_i \gro_j}+ \sum_{m=1}^{t-1} \prod_{n=m}^{t-1} \Big(\poo^{m+1}_{\gro_i\gro_j}-\pzo^{m+1}_{\gro_i\gro_j}\Big) \pzo^{m}_{\gro_i\gro_j}+\pzo^{t}_{\gro_i\gro_j}
\end{align}
where $\wt^1=\wrs^1$. Equation (\ref{SBM_prf}) shows that the each graph marginal retains the same SBM structure because the link formation probability depends on the community memberships of only the relevant nodes, which are fixed across time.  
\end{proof}

\subsection{Proof of Lemma~\ref{Unfiycm_correct}}
\begin{proof}
If there are two community membership vectors $g_1$ and $g_2$ such that $\sigma(g_2) = g_1$ for some permutation $\sigma$, then it is easy to see that $Y_1 = g_1g_1'$ and $Y_2 = g_2g_2'$ are equal.

Thus, if $\hat{\gro}^t= \gro^*$ for all $1 \leq t 
\leq T$, then we have: 
\[\hat{\y}^t=(\hat{\gro}^t) (\hat{\gro}^t)^\prime= (\hat{\gro}^*) (\hat{\gro}^*)^\prime= \y^*, \quad 1 \leq t \leq T.\]
We also have $\epsilon = 0$ (see Algorithm~\ref{unify cms}).

Thus, the output of \textsc{UnifyCM} is $\hat{\y}:=\argmin_{\y: \text{cluster matrix}} \norm{\y - \frac{1}{T}\sum_{t=1}^T \y^*}_F= \y^*$. Hence, the algorithm outputs $\gro^*$ (which can be deduced from $\hat{\y}=\y^*=(\hat{\gro}^*) (\hat{\gro}^*)^\prime$) as expected.
\end{proof}

\subsection{Proof of Theorem~\ref{UnifyCM_Grnt}}
\begin{proof}

Given \textbf{Assumption A}, for a pair of nodes $(i,j)$ at time $t$, the corresponding entry in the cluster matrix $\hat{\y}^t$ can be erroneous due to the following two reasons:
\begin{enumerate}
\item $\gro_i^*= \gro_j^*$, but $\hat{\gro}_i^t \neq \hat{\gro}_j^t.$ The probability of this error, defined as assigning $i,j$ to different communities erroneously, is: $\p=1-\frac{\eps^2}{\noc-1}-(1-\eps)^2$.
\item $\gro_i^*\neq \gro_j^*$, but $\hat{\gro}_i^t= \hat{\gro}_j^t.$ The probability of this error, defined as assigning $i,j$ to the same communities erroneously, is: $\q= \frac{\eps^2}{\noc-1}$.
\end{enumerate}

Thus, we note that $\hat{\y}_{ij}^t= \begin{cases}
\y^*_{ij}, \quad w.p. \quad (1-\p)\y^*_{ij}+(1-\y_{ij}^*)(1-\q)\\
\y^*_{ij}, \quad w.p. \quad \p \y_{ij}^*+(1-\y_{ij}^*)\q,
\end{cases}$ 
where $\y^*$ is the true solution and $\hat{\y}^t$ is a cluster matrix representation of community membership vector $\hat{\gro}^t \in [\noc]^{\no}$.

This implies: 
\[\E[\hat{\y}_{ij}^t]= \y_{ij}^* (1-\p)+(1-\y_{ij}^*)\q= (1-\eps)^2\y^*_{ij}+\q.\]
Also, 
\[\norm{\hat{\y}^t-\E[\hat{\y}^t]}=\norm{\hat{\y}^t-(1-\eps)^2\y^*_{ij}+\q}\leq \no+(1-\eps)^2 \max_{l \in [\noc]} k_l +q,\]
where $k_l$ is the size of community $l \in [\noc]$.

Let $R:= \frac{1}{T}(\no+(1-\eps)^2 \max_{l \in [\noc]} k_l +q).$
Define $\sigma^2$ as:
\begin{align*}
\sigma^2 &= \frac{1}{T^2} \norm{\sum_{t=1}^T \E(\hat{\y}^t-\E[\hat{\y}^t])^2}
\leq \frac{1}{T^2} \sum_{t=1}^T\norm{\E(\hat{\y}^t-\E[\hat{\y}^t])^2}
\leq \frac{1}{T^2}\sum_{t=1}^T \norm{\E[\hat{\y}^t]} + \frac{1}{T^2}\norm{\E^2[\hat{\y}^t]}, \\
&\leq \frac{1}{T^2} \sum_{t=1}^T (\norm{\E[\hat{\y}^t]}+\norm{\E[\hat{\y}^t]}\norm{\E[\hat{\y}^t]})
= \frac{1}{T}\Big((1-\eps)^2\max_{l \in [\noc]} k_l+\q+\big( (1-\eps)^2\max_{l \in [\noc]} k_l+\q\big)^2\Big),
\end{align*}
where $\lVert\cdot \rVert$ is the spectral norm. Using the matrix Bernstein theorem from \citet{tropp2012user}, we have:
\begin{align}\label{mat_bernstein}
P\Big(\norm{\frac{1}{T}\sum_{t=1}^T\hat{\y}^t-\E[\hat{\y}^t]}\geq u \Big) \leq n^2 \exp(\frac{-u^2/2}{\sigma^2+Ru/3}), \quad \forall u\geq 0.
\end{align}
Note that:
\begin{align*}
\norm{\frac{1}{T}\sum_{t=1}^T(\hat{\y}^t)-\E[\hat{\y}^t]}
&=\norm{\frac{1}{T}\sum_{t=1}^T(\hat{\y}^t)-(1-\eps)^2\y^*+\q \mathbf{1}\mathbf{1}'} 
\leq \norm{\frac{1}{T}\sum_{t=1}^T(\hat{\y}^t)-(1-\eps)^2\y^*}_F +q\\
&= (1-\eps)^2\norm{\frac{1}{T(1-\eps)^2}\sum_{t=1}^T(\hat{\y}^t)-\y^*}_F +q 
\end{align*}
Let $\bar{{\y}}:=\frac{1}{T(1-\eps)^2}\sum_{t=1}^T(\hat{\y}^t)$, and define $\hat{\y}:=\argmin_{\{y: \text{ cluster matrix}\}} \norm{\y- \frac{1}{T(1-\eps)^2}\sum_{t=1}^T(\hat{\y}^t)}_F$.
Note that:
$\norm{\bar{\y}-\y^*}_F= \norm{\bar{\y}-\hat{\y}+\hat{\y}-\y^*}_F\leq \norm{\bar{\y}-\hat{\y}}_F+\norm{\hat{\y}-{\y^*}}_F$.\\
Assume $\norm{\bar{\y}-\hat{\y}}_F\leq B.$ Now from Equation (\ref{mat_bernstein}), we get:
\begin{align*}
P\Big((1-\eps)^2 (\norm{\hat{\y}-\y^*}_F+B)+q \geq u \Big) \leq n^2 \exp(\frac{-u^2/2}{\sigma^2+Ru/3})
\end{align*}
Thus, we get:
\begin{align*}
P\Big(\norm{\hat{\y}-\y^*}_F \geq \frac{u-q}{(1-\eps)^2}-B \Big) \leq n^2 \exp(\frac{-u^2/2}{\sigma^2+Ru/3}), \quad \forall u\geq 0.
\end{align*}
\end{proof}

\subsection{Proof of Lemma \ref{SSBM_Grnt}}
\begin{proof}
Note that $\hat{\y}^t= [z_{1t} ... z_{\no t}]' [z_{1t} ... z_{\no t}]$. Define $f(z_{1t}, ..., z_{\no t})= \sum_{i=1}^{\no} \sum_{j=1}^{\no} (z_{it}'z_{jt}-\y^*_{ij})^2$. Also,
\begin{align*}
\sup_{z_{1t}, ... ,z_{\no t},z^{new}_{it}} \abs{f(z_{1t}, ...,z_{it},..., z_{\no t})-f(z_{1t}, ...,z^{new}_{it},..., z_{\no t})}\leq 2\no, \quad \forall i=1,...,\no.
\end{align*}
Now, using the bounded differences inequality, for any $u_1\geq 0$, we have:
\[ P\Big(f(z_{1t}, ..., z_{\no t})\geq \E [f(z_{1t}, ..., z_{\no t})] + u_1\Big)\leq \exp(\frac{-2u_1^2}{\sum_{i=1}^{\no} (2\no)^2})= \exp(\frac{-2u_1^2}{4\no^3})\]
also, $\E [f(z_{1t}, ..., z_{\no t})]$ can be computed as follows:
\begin{align*}
\E [f(z_{1t}, ..., z_{\no t})]&=\E\Big[\sum_{i=1}^{\no} \sum_{j=1}^{\no} (z_{it}'z_{jt}-\y^*_{ij})^2\Big]=\sum_{i=1}^{\no} \sum_{j=1}^{\no} \E\Big[(z_{it}'z_{jt})^2+a^2-2az_{it}'z_{jt}\Big]
\\&= \E[(z_{it}'z_{jt})^2]+a^2-2a\E[z_{it}']\E[z_{jt}],
\end{align*}
where in the last equality $a=\y^*_{ij}$, and $z_{it} \perp z_{jt}.$ Also, let $p_l:=\E[z_{it}]= \begin{cases}
1-\eps \quad if \quad l=\gro_i^*\\
\eps/(\noc -1) \quad o.w.
\end{cases}.$ Note that 
\[\E[(z_{it}'z_{jt})^2]=\E[\E[(z_{it}' z_{jt})^2|z_{jt}=\beta]]= \E[\E[(\sum_l \beta_l z_{ilt})^2|z_{jt}=\beta]].\] 
The summation inside can be expressed as:
\[\sum_l \beta_l^2 z_{ilt}^2+ 2\sum_{l<m}\beta_l \beta_m z_{ilt}z_{imt}.\]
Substituting the expression for $\beta$ and the following values in the expected value of the summation above, viz., \[\E[z_{ilt}^2]=\text{Var}[z_{ilt}]-\E[z_{ilt}]^2=p_l(1-p_l)-p_l^2=p_l-2p_l^2,\]
\[\E[z_{ilt}z_{imt}]=p_lp_m,\]
we get: 
\[\E[(z_{it}'z_{it})^2]=\E[\sum_l z_{ilt}^2(p_l-2p_l^2)+2\sum_{l<m} z_{ilt} z_{imt} p_l p_m] = \sum_l (q_l-2q_l^2)(p_l-2p_l^2) + 2\sum_{l<m} q_l q_m p_l p_m.
\]
\end{proof}

\subsection{Proof of Lemma~\ref{claim:aac}}

\begin{proof}
First, we apply $\per_t$ on input $\hat{\gro}^t$ to get $g^*$ for each time index $1\leq t \leq T$. If we give these $g^*$s as inputs to \textsc{UnifyLP} (Algorithm~\ref{unify LP}), then $Q^t=Q^*$  for all $1 \leq t \leq T$. Naturally $\tau^t=\mathbb{I}$ (an identity matrix) is the unique optimal solution. As a consequence, for each $i \in [\no]$, we have $\hat{\gro}_i=\argmax_{k \in [\noc]}\sum_{t \in [T]}e_i\trn{Q^*}=\argmax_{k \in [\noc]} T e_i\trn{Q^*}={\gro}^*_i$, as desired.
\end{proof}

\subsection{Proof of Theorem~\ref{2SBM_thm}}

\begin{proof}
First, marginalizing over $\adj_{ij}^{t-1}$ we get:
\begin{align}\label{magin_t-1}
\wrs_{ij}^t(a,b)&=\xi P(\adj_{ij}^{t-1}=1|\gro_i^t=a, \gro_j^t=b, \maj)+(1-\xi)\wrs^t_{ab}.
\end{align}
Next, summing over the values that $\gro_i^{t-1} and \gro_j^{t-1}$ take, we get:
\begin{align*}
P(\adj_{ij}^{t-1}|\gro^t, \maj)& = 
\sum_{\gro_i^{t-1},\gro_j^{t-1}} P(\adj_{ij}^{t-1}=1|\gro^t,\gro_i^{t-1},\gro_j^{t-1}, M) \prod_{s=i,j} \Big[\maj_{s}^{t-1}\delta_{\gro_{s}^t, \gro_{s}^{t-1}}+\bar{\maj}_{s}^{t-1} \bar{\delta}_{\gro_{s}^t,\gro_{s}^{t-1}}\frac{1}{\noc-1}\Big],\\
&=\sum_{\gro_i^{t-1},\gro_j^{t-1}} P(\adj_{ij}^{t-1}=1|\gro_i^{t-1},\gro_j^{t-1}, M) \Big[\maj_{i}^{t-1}\delta_{\gro_{i}^t, \gro_{i}^{t-1}}+\bar{\maj}_{i}^{t-1} \bar{\delta}_{\gro_{i}^t, \gro_{i}^{t-1}}\frac{1}{\noc-1}\Big]\\
&\hspace{1cm}\times\Big[\maj_{j}^{t-1}\delta_{\gro_{j}^t, \gro_{j}^{t-1}}+\bar{\maj}_{j}^{t-1} \bar{\delta}_{\gro_{j}^t, \gro_{j}^{t-1}}\frac{1}{\noc-1}\Big],\\
&= \Big[\maj_i^{t-1}\maj_j^{t-1}\wrs_{ij}^{t-1}(a,b)\Big]+\Big[\frac{\maj_i^{t-1}\bar{\maj}_j^{t-1}}{\noc-1} \sum_{\gro_j^{t-1}\neq b} \wrs_{ij}^{t-1}({a,\gro_j^{t-1}})\Big]\\
&+\Big[\frac{\bar{\maj}_i^{t-1}\maj_j^{t-1}}{\noc-1}\sum_{\gro_i^{t-1}\neq a}\wrs_{ij}^{t-1}({\gro_i^{t-1},b})\Big]+\Big[\frac{\bar{\maj}_i^{t-1}\bar{\maj}_j^{t-1}}{(\noc-1)^2}\sum_{\gro_i^{t-1}\neq a,\gro_j^{t-1}\neq b}\wrs_{ij}^{t-1}({\gro_i^{t-1},\gro_j^{t-1}})\Big].
\end{align*}
Substituting this in equation (\ref{magin_t-1}), we get:
\begin{align}
\wrs_{ij}^t(a,b)&= \xi \Bigg[\Big[\maj_i^{t-1}\maj_j^{t-1}\wrs_{ab}^{t-1}\Big]+\Big[\frac{\maj_i^{t-1}\bar{\maj}_j^{t-1}}{\noc-1} \sum_{\gro_j^{t-1}\neq b} \wrs_{a\gro_j^{t-1}}^{t-1}\Big]\\
&+\Big[\frac{\bar{\maj}_i^{t-1}\maj_j^{t-1}}{\noc-1}\sum_{\gro_i^{t-1}\neq a}\wrs_{\gro_i^{t-1}b}^{t-1}\Big]+\Big[\frac{\bar{\maj}_i^{t-1}\bar{\maj}_j^{t-1}}{(\noc-1)^2}\sum_{\gro_i^{t-1}\neq a,\gro_j^{t-1}\neq b}\wrs_{\gro_i^{t-1}\gro_j^{t-1}}^{t-1}\Big] \Bigg]\nonumber\\
&+ (1-\xi) \wrs^t_{ab}.\nonumber
\end{align}
\end{proof}

\subsection{Proof of Theorem~\ref{General_Min_label}}
\begin{proof}

When nodes $i$ and $j$ are both majority at time index $t-1$, substituting this information in Equation~(\ref{2_SBM}) gives us the  probability of link formation at time $t$ (this is equal to $\wrs^t_{ab}$ if $\gro_{i}^{t-1} = a$ and $\gro_{j}^{t-1} = b$). Now, we can observe that if both the nodes $i$ and $j$ are majority (i.e., $\maj_i^{t=1} = \maj_j^{t=1} = 1$), then the marginal probability of edge formation depends on the community memberships of $i$ and $j$ and not other nodes. Hence, the majority nodes retain an SBM structure. Similarly, we can see that if $i$ and $j$ are both minority, a similar property holds. Thus, we obtain two sub-graphs with each retaining the SBM sub-structure.

Now, consider the setting when one of the nodes is a minority and the other is a majority at time $t-1$. In this case, we can compare the coefficients of $\maj^{t-1}_i\bar{\maj}^{t-1}_j$ or $\bar{\maj}^{t-1}_i\maj^{t-1}_j$ in Equation~(\ref{2_SBM}) with empirical estimates to check if they belong to the same community (i.e., $a=b$) at time $t$.  Without loss of generality, let $i$ be a majority node at time $t-1$, and let $j$ be a minority node. Since till time $t-1$ all nodes were majority, from Equation~(\ref{2_SBM}), we have the following relation:

\[\wrs_{ij}^t(a,b)=\frac{\xi}{\noc-1} \sum_{\gro_j^{t-1}\neq b} \wrs_{a\gro_j^{t-1}}^{t-1}+(1-\xi)\wrs^t_{ab}\]

Based on the above, we can deduce that the community memberships estimated separately for the minority and majority sub-graphs at time $t$ can be aligned, allowing us to recover the community memberships of the minorities, as follows:

\begin{enumerate}
\item $\wrs_{ij}^t(a,a)>\wrs_{ij}^t(a,b) \quad \forall b\neq a \Rightarrow (1-\xi)(\wrs^t_{aa}-\wrs^t_{ab})+\frac{\xi}{\noc-1}(\wrs^{t-1}_{ab}-\wrs^{t-1}_{aa})>0  \quad \forall b\neq a$
\item $\wrs_{ij}^t(a,a)<\wrs_{ij}^t(a,b) \quad \forall b\neq a
\Rightarrow (1-\xi)(\wrs^t_{aa}-\wrs^t_{ab})+\frac{\xi}{\noc-1}(\wrs^{t-1}_{ab}-\wrs^{t-1}_{aa})<0 \quad \forall b\neq a$
\end{enumerate}

The first case implies that the probability of link formation between $i,j$ when $b = a$ is the highest. We can estimate the number of links across communities of the two SBM structures and align those two communities that have the highest number of links between each other. Similarly, in the second case, we can do the opposite and align the two SBM structures. Thus, in the both of above cases, we are able to recover the communities of all nodes at time $t$.

\end{proof}

\subsection{Proof of Lemma~\ref{claim baruca marginal}}

\begin{proof}
\begin{enumerate}
\item We use induction to prove this claim for any value of $t$. By assumption of uniform prior on community assignment at the initial time step, we have $P(\gro_i^1=a|\maj)=\frac{1}{\noc}, \quad \forall a \in [\noc]$. Next, assuming  that for any $b \in [\noc],$ $P(\gro_i^t=b|\maj)=\frac{1}{\noc}$, we prove that $P(\gro_i^{t+1}=a|\maj)=\frac{1}{\noc}, \quad \forall a \in [\noc]$. Indeed,
\begin{align*}
P(\gro_i^{t+1}=a|\maj)&= \sum_{b=1}^{\noc}P(\gro_i^{t+1}=a, \gro_i^t=b|\maj),\\
&= \sum_{b=1}^{\noc} P(\gro_i^{t+1}=a|\gro_i^{t}=b, \maj) P(\gro_i^t=b|\maj),\\
&= \frac{1}{\noc}\big[\maj^t_i \times 1+(1-\maj^t_i)\frac{\noc-1}{\noc-1}\big]=\frac{1}{\noc}.
\end{align*}
\item We prove the result for a $t$ (such that all nodes don't change their communities till time index $t-1$) as follows:
\begin{align*}
P(\adj_{ij}^{t}=1|\maj)&= \sum_{\alpha\in \{0,1\}, \gro^{t}_i, \gro^{t}_j,\gro^{t-1}_i,\gro^{t-1}_j}P(\adj_{ij}^{t}=1|\adj_{ij}^{t-1}=\alpha,\gro_i^{t},\gro_j^{t},\maj) P(\adj_{ij}^{t-1}=\alpha|\gro_i^{t-1},\gro_j^{t-1},\maj)\\ 
&\quad\quad\quad\quad \times P(\gro_i^{t}|\gro_i^{t-1},\maj)P(\gro_j^{t}|\gro_j^{t-1},\maj) P(\gro_i^{t-1}|\maj)P(\gro_j^{t-1}|\maj),\\
&=\frac{1}{\noc^2}\sum_{\gro_i^{t},\gro_j^{t},\gro_i^{t-1},\gro_j^{t-1}} \Big[(1-\xi)\wrs_{\gro_i^{t},\gro_j^{t}}(1-\wrs_{\gro_i^{t-1},\gro_j^{t-1}}^{t-1})+\Big(\xi+(1-\xi)\wrs_{\gro_i^{t},\gro_j^{t}}\Big)\wrs_{\gro_i^{t-1},\gro_j^{t-1}}^{t-1}\Big]\\
&\quad\quad\quad\quad \times\Big[\delta_{\gro_i^{t},\gro_i^{t-1}}\maj_i^{t-1}+\frac{1}{\noc-1}\bar{\delta}_{\gro_i^{t},\gro_i^{t-1}}\bar{\maj}_i^{t-1}\Big]\Big[\delta_{\gro_j^{t},\gro_j^{t-1}}\maj_j^{t-1}+\frac{1}{\noc-1}\bar{\delta}_{\gro_j^{t},\gro_j^{t-1}}\bar{\maj}_j^{t-1}\Big],\\
&\stackrel{(*)}{=}\frac{1}{\noc^2}\sum_{\gro_i^{t},\gro_j^{t},\gro_i^{t-1},\gro_j^{t-1}} \Big[\xi\wrs_{\gro_i^{t-1}\gro_j^{t-1}}+(1-\xi)\wrs_{\gro_i^{t}\gro_j^{t}}\Big] \times T_2,\\
&= \frac{1}{\noc^2} \xi \sum_{\gro_i^{t-1},\gro_j^{t-1}} \wrs_{\gro_i^{t-1}\gro_j^{t-1}} \sum_{\gro_i^{t},\gro_j^{t}} T_2 + \frac{1}{\noc^2} (1-\xi)\sum_{\gro_i^{t},\gro_j^{t}}\wrs_{\gro_i^{t}\gro_j^{t}} \sum_{\gro_i^{t-1},\gro_j^{t-1}} T_2, \\&
\end{align*}
where in equality $(*)$, we use the result from Section~\ref{sec:group-fixed-algo} Equation (\ref{Barrucca W recursion}), namely that $\wrs^{t-1}_{\gro_i^{t-1},\gro_j^{t-1}}= \wrs_{\gro_i^{t-1},\gro_j^{t-1}}$. Also, for simplicity we substitute 
\begin{align*}
T_2 &= \maj_i^{t-1}\maj_j^{t-1}\delta_{\gro_i^{t},\gro_i^{t-1}}\delta_{\gro_j^{t},\gro_j^{t-1}}
+\frac{\maj_i^{t-1}\bar{\maj}_j^{t-1}}{\noc-1}\delta_{\gro_i^{t},\gro_i^{t-1}}\bar{\delta}_{\gro_j^{t},\gro_j^{t-1}}\\
&\quad\quad\quad +\frac{\bar{\maj}_i^{t-1} \maj_j^{t-1}}{\noc-1}\bar{\delta}_{\gro_i^{t},\gro_i^{t-1}}{\delta}_{\gro_j^{t},\gro_j^{t-1}}+\frac{\bar{\maj}_i^{t-1} \bar{\maj}_j^{t-1}}{(\noc-1)^2}\bar{\delta}_{\gro_i^{t},\gro_i^{t-1}}\bar{{\delta}}_{\gro_j^{t},\gro_j^{t-1}}.
\end{align*}
There are four possible values that the pair $(M_i^{t-1}, M_j^{t-1})$ can take. In each setting, we can show that the summations simplify giving us the desired result $P(\adj_{ij}^{t}=1|\maj)=\bar{\wrs}$.

\item Again we prove the result for a $t$ (such that all nodes don't change their communities till time index $t-1$) as follows:
\begin{align*}
\frac{1}{\noc}\sum_{a} \wrs^{t}_{ij}(a,b)&=\frac{1}{\noc} \sum_{a} P(\adj_{ij}^{t}=1|\gro_i^{t}=a, \gro_j^{t}=b, \maj),\\
&= \frac{1}{\noc} \sum_{a,\gro^t, \beta}P(\adj^{t}=1|\adj^{t-1}=\beta, \gro^{t}=(a,b), \gro^{t-1}, \maj)\\
&\hspace{1cm}\times P(\adj^{t-1}=\beta|\gro^{t}=(a,b), \gro^{t-1}, \maj) P(\gro^{t-1}|\gro^{t}, \maj),\\
&\stackrel{(*)}{=}\frac{1}{\noc} \sum_a \sum_{\beta\in\{0,1\}}\sum_{\gro^{t-1}} (\xi\delta_{\beta,1}+(1-\xi)\wrs_{ab}) \wrs_{\gro_i^{t-1},\gro_j^{t-1}}^\beta (1-\wrs_{\gro_i^{t-1},\gro_j^{t-1}})^{(1-\beta)} T_2,\\
&=\frac{1}{\noc}\sum_a\Big[(\xi+(1-\xi)\wrs_{ab})(\sum_{\gro^{t-1}=(c,d)}\wrs_{cd}T_2)+(1-\xi)\wrs_{ab}(\sum_{cd}T_2-\sum_{cd}\wrs_{cd}T_2)\Big],\\
&=
\frac{1}{\noc}\sum_a\Big[\xi\sum_{c,d}\wrs_{cd}T_2+(1-\xi)\wrs_{ab}\sum_{c,d}T_2\Big],
\end{align*}
where in equality $(*)$ we use the result from Section~\ref{sec:group-fixed-algo} Equation (\ref{Barrucca W recursion}), namely that $\wrs^{t-1}_{\gro_i^{t-1},\gro_j^{t-1}}= \wrs_{\gro_i^{t-1},\gro_j^{t-1}}$. Also, for simplicity we substitute $T_2= \maj_i^{t-1}\maj_j^{t-1}\delta_{\gro_i^{t},\gro_i^{t-1}}\delta_{\gro_j^{t},\gro_j^{t-1}}+\frac{\maj_i^{t-1}\bar{\maj}_j^{t-1}}{\noc-1}\delta_{\gro_i^{t},\gro_i^{t-1}}\bar{\delta}_{\gro_j^{t},\gro_j^{t-1}}+\frac{\bar{\maj}_i^{t-1} \maj_j^{t-1}}{\noc-1}\bar{\delta}_{\gro_i^{t},\gro_i^{t-1}}{\delta}_{\gro_j^{t},\gro_j^{t-1}}+\frac{\bar{\maj}_i^{t-1} \bar{\maj}_j^{t-1}}{(\noc-1)^2}\bar{\delta}_{\gro_i^{t},\gro_i^{t-1}}\bar{{\delta}}_{\gro_j^{t},\gro_j^{t-1}}$. Considering different cases for $\maj_i^{t-1}$ and $\maj_j^{t-1}$, we get $\frac{1}{\noc}\sum_{a} \wrs_{ij}^t(a,b)=\bar{\wrs}, \quad \forall b\in [\noc]$.
\end{enumerate}

\end{proof}

\subsection{Proof of Corollary~\ref{prop baruca marginal}}

\begin{proof}
First, marginalizing over $\adj_{ij}^{t-1}$ we get:
\begin{align*}
\wrs_{ij}^t(a,b)&=\xi P(\adj_{ij}^{t-1}=1|\gro_i^t=a, \gro_j^t=b, \maj)+(1-\xi)\wrs_{ab}.
\end{align*}
Next, summing over the values that $\gro_i^{t-1} and \gro_j^{t-1}$ take, we get:
\begin{align*}
P(\adj_{ij}^{t-1}|\gro^t, \maj)& = 
\sum_{\gro_i^{t-1},\gro_j^{t-1}} P(\adj_{ij}^{t-1}=1|\gro^t,\gro_i^{t-1},\gro_j^{t-1}, M) \prod_{s=i,j} \Big[\maj_{s}^{t-1}\delta_{\gro_{s}^t, \gro_{s}^{t-1}}+\bar{\maj}_{s}^{t-1} \bar{\delta}_{\gro_{s}^t,\gro_{s}^{t-1}}\frac{1}{\noc-1}\Big],\\
&=\sum_{\gro_i^{t-1},\gro_j^{t-1}} P(\adj_{ij}^{t-1}=1|\gro_i^{t-1},\gro_j^{t-1}, M) \Big[\maj_{i}^{t-1}\delta_{\gro_{i}^t, \gro_{i}^{t-1}}+\bar{\maj}_{i}^{t-1} \bar{\delta}_{\gro_{i}^t}\frac{1}{\noc-1}\Big]\\
&\hspace{1cm}\times\Big[\maj_{j}^{t-1}\delta_{\gro_{j}^t, \gro_{j}^{t-1}}+\bar{\maj}_{j}^{t-1} \bar{\delta}_{\gro_{j}^t}\frac{1}{\noc-1}\Big],\\
&= \Big[\maj_i^{t-1}\maj_j^{t-1}\wrs_{ij}^{t-1}(a,b)\Big]+\Big[\frac{\maj_i^{t-1}\bar{\maj}_j^{t-1}}{\noc-1} \sum_{\gro_j^{t-1}\neq b} \wrs_{ij}^{t-1}({a,\gro_j^{t-1}})\Big]\\
&+\Big[\frac{\bar{\maj}_i^{t-1}\maj_j^{t-1}}{\noc-1}\sum_{\gro_i^{t-1}\neq a}\wrs_{ij}^{t-1}({\gro_i^{t-1},b})\Big]+\Big[\frac{\bar{\maj}_i^{t-1}\bar{\maj}_j^{t-1}}{(\noc-1)^2}\sum_{\gro_i^{t-1}\neq a,\gro_j^{t-1}\neq b}\wrs_{ij}^{t-1}({\gro_i^{t-1},\gro_j^{t-1}})\Big].
\end{align*}
Using claims from  Lemma~\ref{claim baruca marginal}, viz., $P(\gro_i^t|\maj)=\frac{1}{\noc}$, $P(\adj_{ij}^t=1|\maj)=\bar{\wrs}$, and $\frac{1}{\noc}\sum_{a} \wrs_{ij}^t(a,b)=\bar{\wrs}, \quad \forall b\in [\noc]$ , we get:
\begin{align*}
\wrs_{ij}^t(a,b)&= \xi \Bigg[\maj_i^{t-1} \maj_j^{t-1} \wrs_{ij}^{t-1}(a,b)+ \frac{\maj_i^{t-1}\bar{\maj}_j^{t-1}}{\noc-1}(\noc \bar{\wrs}-\wrs_{ij}^{t-1}(a,b))+\frac{\bar{\maj}_i^{t-1}{\maj}_j^{t-1}}{\noc-1} \nonumber\\&\times(\noc\bar{\wrs}-\wrs_{ij}^{t-1}(a,b))+\frac{\bar{\maj}_i^{t-1}\bar{\maj}_j^{t-1}}{(\noc-1)^2}\Big((\noc^2-2\noc)\bar{\wrs}+\wrs_{ij}^{t-1}(a,b)\Big) \Bigg]+ (1-\xi) \wrs_{ab},\\
&= \xi \Bigg[\maj_i^{t-1} \maj_j^{t-1} \wrs_{ab}+ \frac{\maj_i^{t-1}\bar{\maj}_j^{t-1}}{\noc-1}(\noc \bar{\wrs}-\wrs_{ab})+\frac{\bar{\maj}_i^{t-1}{\maj}_j^{t-1}}{\noc-1} \nonumber\\
&\times(\noc\bar{\wrs}-\wrs_{ab})+\frac{\bar{\maj}_i^{t-1}\bar{\maj}_j^{t-1}}{(\noc-1)^2}\Big((\noc^2-2\noc)\bar{\wrs}+\wrs_{ab}\Big) \Bigg]+ (1-\xi) \wrs_{ab}.
\end{align*}
\end{proof}

\subsection{Proof of corollary~\ref{claim_Min_label}}

\begin{proof}

When nodes $i$ and $j$ are both majority at time index $t-1$, substituting this information in Equation~(\ref{marginal barrucca}) gives us the  probability of link formation at time $t$ (this is equal to $W_{ab}$ if $\gro_{i}^{t-1} = a$ and $\gro_{j}^{t-1} = b$). We can see that this probability does not depend on the community memberships of any other node, except $i$ and $j$. Hence, the resulting subgraph has an SBM structure. A similar observation can be made when both nodes are minority. Thus there are two SBM models at time $t$ when we only consider links within each set and prune the rest. Assuming that we solve for the community memberships for each set, we show below that these memberships can always be aligned with each other, allowing for the consistent recovery of all community memberships at time $t$.

Consider the setting when one of the nodes is a minority and the other is a majority at time $t-1$. In this case, we can compare the coefficient of $\maj^{t-1}_i\bar{\maj}^{t-1}_j$ or $\bar{\maj}^{t-1}_i\maj^{t-1}_j$ in Equation~(\ref{marginal barrucca}) to verify whether they belong to the same community (i.e., $a=b$) at time $t$ or not.  Without loss of generality, let $i$ be a majority node at time $t-1$, and let $j$ be a minority node. Since till time $t-1$ all nodes were majority, from Equation~(\ref{marginal barrucca}), we obtain two cases:

In the first case, when $1 - \xi - \frac{\xi}{k-1} > 0$, 
\begin{align*}
\wrs_{ij}^t(a,b)&=\frac{\xi}{\noc -1}\Big(\noc \bar{\wrs}-\wrs_{ab}\Big) + (1-\xi)\wrs_{ab}
\end{align*}
is the largest when $a = b$. Thus, we can estimate the number of links across communities in the two SBM structures, and align those two communities that have the highest number of links with each other. This is because the parameter matrix $\wrs$ is an assortative matrix, so if we have $a=b$, the probability of having an edge between nodes $i$ and $j$ is the highest.

In the second case, when $1 - \xi - \frac{\xi}{k-1} \leq 0$, the opposite of above can be done to align the two SBM structures. Thus, in both cases, we can recover the communities of all nodes at time $t$.
\end{proof}

\section{A Convex Heuristic based on UnifyCM}
In this section, we discuss a convex programming based approach for estimating community memberships based on \textsc{UnifyCM} called \textsc{ConvexUnify} (Algorithm~\ref{unify cms2}). 

\begin{algorithm}[h!]
\caption{\textsc{ConvexUnify}}\label{unify cms2}
\begin{algorithmic}[1]
\\\quad\textbf{Input:} Community assignments $\hat{\gro}^t$, $t \in [T]$, threshold $C$.
\\\quad\textbf{Output:} Unified community assignment $\hat{\gro}$. 
\\\quad Set $\adj^u=0_{\no,\no}$.
\\\quad Let $\text{count}_{ij}:=\sum_{t} \hat{\y}_{ij}^t \; \forall i,i \in [\noc]$.
\\\quad \textbf{if} count$_{ij}\geq C$
\\\quad\quad $\adj^u_{ij}=1$
\\\quad $\hat{\y}=\argmax$ of Problem (\ref{clstr_opt}).
\\\quad Deduce $\hat{\gro}$ from $\hat{\y}$.
\end{algorithmic}
\end{algorithm}

Under \textbf{Assumption A} outlined in Section~\ref{sec:group-fixed-algo}, we can study certain properties of \textsc{ConvexUnify}. In particular, we observe that the matrix $\adj^u$ can be erroneous in two cases:
\begin{enumerate}
\item $g_i^*=g_j^*$ but count$_{ij}<C$: Let $p_1$ be the probability of $count_{ij}<C$ while the pair of nodes $i,j$ are in the same true community. Then, 
\[p_1:=P(count_{ij}<C| \gro_i^*= \gro_j^*)=\sum_{l=1}^{C-1} \binom{T}{i} (1-\p)^l (\p)^{T-l}.\]

\item $g_i^*\neq g_j^*$ but count$_{ij}\geq C$: Let $p_2$ be the probability of $count_{ij}\geq C$ while the pair of $i,j$ are not in the same true community. Then,
\[p_2:=P(count_{ij}\geq C| \gro_i^*\neq \gro_j^*)=\sum_{l=C}^{T} \binom{T}{i} \q^l (1-\q)^{T-l}.\]
\end{enumerate}
In the above, $\p$ and $\q$ are the same as in the proof of Theorem~\ref{UnifyCM_Grnt}.

Let $\pi$ be the unconditional probability of false observation in matrix $\adj^u$, i.e., that of missing an edge between two nodes that belong to the same true community or observing an edge between two nodes that belong to different true communities. Then, $\pi$ can be computed as $\pi= \frac{1}{\noc^2} p_1 + (1-\frac{1}{\noc^2})p_2$. Further, note that $\pi$ needs to be less than $1/2$, otherwise the observation of the graph $\adj^u$ has no information, and recovering the true community memberships is not possible. 

We have:
$\adj^u_{ij}=
\begin{cases}
\y^*_{ij}, \quad w.p. \quad 1-\pi\\
1-\y^*_{ij}, \quad w.p. \quad \pi.
\end{cases}$\\

Now, the goal of recovering true community membership matrix $\y^*$ from a erroneous adjacency matrix $\adj^u$ can be achieved by solving:
\begin{align*}
&\max \sum_{i,j} (2\adj^u_{ij}-1) \y_{ij} \quad s.t.\\\nonumber
&\quad \y \text{is a cluster matrix.}
\end{align*}

The above combinatorial problem can be relaxed~\cite{chen2014weighted}
 to  Problem (\ref{clstr_opt}) and used in Algorithm~\ref{unify cms2}.  This problem is a semi-definite program (SDP) and can be run in polynomial time:
\begin{align}\label{clstr_opt}
\max &\sum_{i,j} (2\adj^u_{ij}-1) \y_{ij} \quad s.t.\\\nonumber
&\quad \y \in S_{nuclear}, \textrm{ and}\\\nonumber
&\quad 0 \leq \y_{ij} \leq 1,
\end{align}
where $S_{nuclear}$ is a convex set containing $\y^*$ defined below:
\begin{align}\label{S_nuclear}
S_{nuclear}:=\{Y\in \mathbb{R}^{n\times n}: ||Y||_*\leq \no\},
\end{align}
where $||\y||_*$ represents nuclear (trace) norm constraint.

Although we don't give statistical convergence guarantees for \textsc{ConvexUnify}, we can show that it recovers the true community assignment vector $\gro^*$, at least in the case when there are no errors in the inputs. 

\begin{lemma}\label{unifycm2_crct}
\textsc{ConvexUnify} (Algorithm \ref{unify cms2}) finds the true community assignment $\gro^*$ if there are no errors in its input for any $C \in (0,T)$.
\end{lemma}

\begin{proof}
Since there are no errors in the inputs, w.l.o.g, we can assume $\hat{\gro}^t=\gro^*, \quad 1\leq t \leq T.$ Thus, we have:
\[ \text{count}_{ij}=\begin{cases}
T, \quad if \quad \gro_i^*=\gro_j^*,\\
0, \quad o.w.
\end{cases}
\]
This implies: 
\[ \adj^u_{ij}=\begin{cases}
1, \quad if \quad \gro_i^*=\gro_j^*,\\
0, \quad o.w.
\end{cases}
\]
Now, we show that $\y^*$ is the solution of Problem (\ref{clstr_opt}). First, observe that $\y^*$ is a feasible solution since $\y^* \in S_{nuclear}$ and $0 \leq \y^*_{ij} \leq 1$. Next, note that
\[ 2 \adj^u_{ij} -1=\begin{cases}
1, \quad if \quad \gro_i^*=\gro_j^*\\
-1, \quad o.w.
\end{cases}
\]
Clearly, for any $0 \leq \y_{ij} \leq 1$ we have: 
\[\sum_{i,j} \Big[\mathbf{1}_{\gro_i^*=\gro_j^*} + (0\times \mathbf{1}_{\gro_i^*\neq \gro_j^*})\Big] \geq \sum_{i,j} \Big[ \y_{ij} \mathbf{1}_{\gro_i^*=\gro_j^*} + (-\y_{ij}) \mathbf{1}_{\gro_i^*\neq\gro_j^*}\Big],\]
which leads to: 
\[\sum_{i,j} (2\adj^u_{ij}-1) \y^*_{ij} \geq \sum_{i,j} (2\adj^u_{ij}-1) \y_{ij}.\]
Hence, $\y^*$ is the optimal solution of Problem (\ref{clstr_opt}), and Algorithm \ref{unify cms2} outputs the true community assignment $\gro^*$.
\end{proof}